\documentclass[11pt,letterpaper]{article}
\usepackage{booktabs}
\usepackage[letterpaper,margin=1in]{geometry}

\usepackage{authblk}

\usepackage{bm}

\usepackage[multiple]{footmisc}

\usepackage{mlmodern}
\usepackage{inconsolata}

\usepackage{booktabs}
\usepackage{mdframed}

\usepackage{dsfont}

\usepackage{amsfonts,nicefrac}
\usepackage[utf8]{inputenc}
\usepackage[dvipsnames]{xcolor}
\usepackage{graphicx}
\usepackage{multirow}
\usepackage{amsmath}
\usepackage{amsthm}
\usepackage{amssymb}
\usepackage{comment}
\usepackage{cancel}

\usepackage[shortlabels]{enumitem}
\usepackage[style=alphabetic, maxbibnames=15, giveninits=true, maxcitenames=15, natbib=true,
  maxalphanames=10, backend=biber, sorting=nty, backref=true, uniquename=false]{biblatex}
\DefineBibliographyStrings{english}{backrefpage = {page},backrefpages = {pages},}
\DeclareNameAlias{sortname}{given-family}
\renewbibmacro{in:}{%
  \ifentrytype{article}{}{\printtext{\bibstring{in}\intitlepunct}}}
  
\makeatletter

\newtheorem{theorem}{Theorem}[section]
\newtheorem{corollary}{Corollary}[theorem]
\newtheorem{lemma}[theorem]{Lemma}
\newtheorem{proposition}[theorem]{Proposition}

\newtheorem{assumption}{Assumption}[section]
\theoremstyle{remark}
\newtheorem{remark}{Remark}[section]

\usepackage{hyperref}
\hypersetup{
  colorlinks   = true, %
  urlcolor     = blue, %
  linkcolor    = blue, %
  citecolor   = red %
}

\renewcommand{\epsilon}{\varepsilon}

\usepackage{bm}

\DeclareMathOperator*{\E}{\mathbb{E}}

\newcommand\Var{\mathrm{Var}}

\newcommand\Cov{\mathrm{Cov}}

\DeclareMathOperator{\sign}{sign}

\makeatother

\usepackage{algorithm}
\usepackage{algpseudocode}

\usepackage{caption}
\usepackage{subcaption}
\usepackage{mathtools}

\floatstyle{ruled}
\newfloat{subroutine}{htbp}{lok}
\floatname{subroutine}{Subroutine}
\usepackage{tikz}
\usetikzlibrary{decorations.pathreplacing,calc}

\definecolor{comment}{RGB}{2,128, 9}

\usepackage{eqparbox}

\usepackage[capitalise]{cleveref}
\crefname{equation}{}{}

\usepackage{enumitem}

\usetikzlibrary{arrows.meta, calc, positioning}

\definecolor{softblue}{rgb}{0.0,0.4,0.9}

\title{Adaptive Optimization via Momentum on Variance-Normalized Gradients
\vspace{3mm}
}

\renewcommand*{\Affilfont}{\normalsize}

\renewcommand*{\Affilfont}{\normalsize}

\makeatletter
\renewcommand\AB@affilsepx{\quad \protect\Affilfont}
\renewcommand\AB@affilsep{\protect\Affilfont}
\makeatother

\author[$\dagger$]{Francisco Patitucci}
\author[$\dagger$,$\ddagger$]{Aryan Mokhtari}

\affil[$\dagger$]{UT Austin}
\affil[$\ddagger$]{Google Research}

\date{}

\date{}

\begin{document}

\maketitle

\begin{abstract}

We introduce {MVN-Grad} (\textit{Momentum on Variance-Normalized Gradients}), an Adam-style optimizer that improves stability and performance by combining two complementary ideas: variance-based normalization and momentum applied \emph{after} normalization. MVN-Grad scales each coordinate by an exponential moving average of gradient uncertainty and applies momentum to the resulting normalized gradients, removing the cross-time coupling between stale momentum and a stochastic normalizer present in standard Adam-type updates. 
We prove that this decoupling yields smaller one-step conditional update variance than momentum--then--normalize variance methods, and that MVN-Grad has a uniformly bounded response to isolated gradient spikes.
In low-variance regimes, we further show that variance normalization avoids sign-type collapse of second-moment scaling and can yield accelerated convergence. Beyond these comparisons, we prove a general nonconvex convergence guarantee for MVN-Grad under bounded-gradient stochastic assumptions. On CIFAR-100 and GPT-style language modeling, MVN-Grad matches or improves on Adam, AdaBelief, and LaProp, delivering smoother training and better generalization at the cost of one additional state tensor.

\end{abstract}

\newpage

\section{Introduction}

Adaptive gradient methods such as Adam \citep{kingma2014adam} are widely used due to their empirical robustness and ease of hyperparameter tuning. Predecessors like AdaGrad \citep{duchi2011adaptive} and RMSProp \citep{hinton2012neural} pioneered the use of coordinate-wise preconditioning based on running statistics. Adam builds on these foundations by combining an exponential moving average (EMA) of gradients with per-coordinate normalization based on the second moment, enabling rapid convergence across diverse architectures. Despite their success, a growing body of work has identified failure modes and generalization limitations of Adam-type methods in stochastic settings \citep{wilson2017marginal, reddi2019convergence, chen2018convergence}.
We identify two fundamental structural limitations in the Adam design that contribute to these failures: temporal coupling and suboptimal scaling.

First, a critical weakness lies in Adam's standard \emph{momentum-then-normalize} ordering, which creates a stochastic coupling between the numerator and denominator. While both terms utilize EMAs, their interaction induces a temporal mismatch. The momentum term $m_t$ acts as a heavy memory buffer for \emph{past} gradients, whereas the adaptive denominator is continuously modulated by the \emph{current} stochastic gradient $g_t$. 
This creates a risk of ``cross-time'' instability: if the normalizer transiently dips---for
instance, during a mini-batch with low gradient magnitude---while the momentum buffer still retains
a large value from a past outlier, the update can become disproportionately sensitive to the current
normalizer. In regimes where the denominator forgets faster than the numerator, this carry-over can
lead to spike-dependent responses.
This coupling amplifies stale gradients and distorts the update direction, a failure mode that has been also highlighted in some recent work \citep{ziyin2020laprop, zhou2018adashift, zhuang2021momentum}.

Second, reliance on the \emph{uncentered second moment} for normalization can lead to ineffective scaling, since $\mathbb{E}[g_t^2]=(\mathbb{E}[g_t])^2+\mathrm{Var}(g_t)$ mixes signal and noise. In low-variance regimes where the signal dominates, the normalizer tracks the squared gradient magnitude, causing the adaptive term to scale the gradient by roughly its own size ($g_t/\sqrt{g_t^2}\approx\mathrm{sign}(g_t)$) and degenerate into a coordinate-wise sign operation. This “sign collapse” discards scale information, forcing uniform step sizes regardless of local curvature and significantly slowing convergence in high-signal regimes \citep{balles2018dissecting, bernstein2018signsgd}.

Past attempts to improve Adam tackle these axes individually, trading one limitation for another. To resolve the temporal mismatch, LaProp \citep{ziyin2020laprop} applies momentum to the \emph{normalized} gradient, while AdaShift \citep{zhou2018adashift} decorrelates numerator and denominator via temporally shifted statistics. Although these methods decouple history from current noise, they retain the second-moment normalizer and remain susceptible to sign collapse. Conversely, AdaBelief \citep{zhuang2020adabelief} adopts variance-based scaling but ignores the ordering flaw; because variance estimates are smaller and more volatile than second moments, the standard ordering can make them less stable than Adam \citep{yuan2020eadam}. Consequently, the design space remains fragmented, forcing a choice between structural stability and adaptive sensitivity.

\noindent\textbf{Contributions.}
We introduce MVN-Grad (Momentum on Variance-Normalized Gradients), which addresses two key limitations via a normalize-then-momentum structure and variance-based normalization.
Applying momentum \emph{after} normalization removes stochastic coupling and improves stability when the normalizer fluctuates, while replacing the second moment with a variance estimate avoids sign collapse and preserves gradient scale in high-signal regimes.
Crucially, these two axes are \emph{orthogonal}: prior methods fix one flaw while leaving the other unresolved. LaProp and AdaShift correct the ordering but retain second-moment normalization (and thus sign collapse); AdaBelief adopts variance-based scaling but retains the momentum-then-normalize ordering (and thus stochastic coupling). MVN-Grad is the principled completion of this design space, resolving both failure modes. %

\begin{itemize}[leftmargin=*, itemsep=4pt, topsep=2pt, parsep=0pt]
    \item \textbf{Theoretical Stability Analysis.} We prove that the \emph{normalize-then-momentum} ordering yields smaller one-step conditional update variance than momentum--then--normalize variance methods such as AdaBelief, with extensions beyond exact symmetry (Theorem~\ref{thm:var_adabelief_mvngrad}; Proposition~\ref{prop:approx_mean}). We also show that MVN-Grad has a bounded response to large gradient spikes under the stated spike conditions (Theorem~\ref{thm:spike}; Corollary~\ref{cor:multi_spike}).

    \item \textbf{High-Signal Convergence.} We show that second-moment methods such as LaProp can collapse to sign-like updates in low-noise regimes, while MVN-Grad preserves gradient magnitudes and avoids the resulting dimension-dependent slowdowns (Theorem~\ref{thm:separation_MVN-Grad_LaProp_smallvar}).

    \item \textbf{General Nonconvex Convergence.} We prove a general convergence guarantee for MVN-Grad under bounded-gradient stochastic nonconvex assumptions, recovering the standard $O(1/\sqrt{T})$ rate up to the cost of variance normalization (Theorem~\ref{thm:mvngrad_convergence}).

    \item \textbf{Empirical Robustness.} On CIFAR-100 and GPT-style language modeling \citep{he2016deep, krizhevsky2009learning, radford2019language, karpathy_nanogpt_2022}, MVN-Grad is consistently competitive with Adam(W), AdaBelief, and LaProp, often improving robustness across hyperparameter sweeps with one additional optimizer-state tensor.

\end{itemize}

\textbf{Additional related work.}
A large body of work has proposed Adam variants targeting stability and convergence (e.g., AMSGrad \citep{reddi2019convergence}, Yogi \citep{zaheer2018adaptive}, RAdam \citep{liu2019variance}), practical training refinements (e.g., AdamW \citep{loshchilov2017fixing}, AdaFactor \citep{shazeer2018adafactor}), or stronger preconditioning (e.g., Sophia \citep{liu2023sophia}, AdaHessian~\citep{yao2021adahessian}).
These methods address complementary aspects of adaptive methods; we review them in Appendix~\ref{app:related-work}.

\section{Proposed Algorithm}
\label{sec:mvn-grad}

\begin{algorithm}[t]
\caption{MVN-Grad (Momentum on Variance-Normalized Gradients)}
\label{alg:mvn-grad}
\begin{algorithmic}[1]
\Require step size $\eta$, $\beta_1,\beta_2\in[0,1)$, $\varepsilon>0$, $\varepsilon_s\ge 0$
\State Initialize $m_0\leftarrow 0$, $s_0\leftarrow 0$, $u_0\leftarrow 0$, parameters $x_0$
\For{$t=1,2,\dots$}
    \State $g_t \gets \nabla_x f(x_{t-1}; \xi_t)$
    \State $m_t \gets \beta_1 m_{t-1} + (1-\beta_1)\, g_t$
    \State $s_t \gets \beta_2 s_{t-1} + (1-\beta_2)\,(g_t - m_t)^2 + \varepsilon_s$
    \State $c_m \gets 1-\beta_1^t$;\;\; $c_v \gets 1-\beta_2^t$
    \State $z_t \gets g_t / (\sqrt{\frac{s_t}{c_v}}+\varepsilon)$
    \State $u_t \gets \beta_1 u_{t-1} + (1-\beta_1)\, z_t$
    \State $x_t \gets x_{t-1} - \eta\, u_t / c_m$
\EndFor
\end{algorithmic}
\end{algorithm}

Consider the stochastic optimization of a differentiable function $F(x)$, where at iteration $t$ we observe a stochastic gradient $g_t=\nabla_x f(x_{t-1};\xi_t)$ with $\xi_t\sim\mathcal{D}$. Updates are element-wise; for readability, we present the scalar form.
We introduce \emph{Momentum on Variance-Normalized Gradients} (MVN-Grad), an Adam-style optimizer defined by two complementary choices: \emph{(i)} variance-based normalization and \emph{(ii)} applying momentum \emph{after} normalization. Concretely, MVN-Grad replaces Adam's second-moment EMA normalizer with a variance estimator, and it applies momentum to the normalized gradients rather than normalizing the momentum itself. We first state the MVN-Grad update, then explain these two ingredients, and finally position MVN-Grad relative to other adaptive methods.

The procedure is summarized in Algorithm~\ref{alg:mvn-grad}. After computing the stochastic gradient $g_t$ (Line 3), we update a gradient EMA, $m_t$, in Line 4. This $m_t$ serves solely to center the second moment for the variance estimation in Line 5. $s_t$ tracks the moving average of squared deviations between $g_t$ and $m_t$, acting as a variance proxy. In Line 7, we use the bias-corrected $s_t$ to normalize the current gradient, producing the direction $z_t$. Finally, we compute the momentum of this normalized sequence, storing it in $u_t$ (Line 8), and use its bias-corrected version to update parameters in Line 9.

\textbf{Implementation notes.}
Algorithm~\ref{alg:mvn-grad} reuses $\beta_1$ for both the mean estimator $m_t$ and the momentum buffer $u_t$ to match Adam's interface; decoupling them may help in non-stationary settings where variance estimation and momentum evolve on different timescales. The constant $\varepsilon_s$ in Line 5 lower-bounds $s_t$, preventing excessive normalization when $(g_t-m_t)^2$ is near zero. Our theory assumes $\varepsilon_s=0$; in practice a very small $\varepsilon_s>0$ improves numerical stability.

\textbf{Taxonomy of Adaptive Architectures.}
Adaptive methods can be categorized along two orthogonal axes. The first axis is the choice of \emph{normalizer}. Standard methods track the uncentered second moment $v_t \approx \mathbb{E}[g_t^2]$, whereas variance-based methods track the centered second moment $s_t \approx \mathrm{Var}(g_t)$. Let $\text{EMA}_{\beta}(x_t)$ denote the exponential moving average of a sequence $x_t$ with decay $\beta$. Using momentum $m_t$ as the mean estimate, the normalizers are:
\begin{equation}\label{v_and_s_definitions}
    v_t = \text{EMA}_{\beta_2}(g_t^2) \quad \text{vs} \quad s_t = \text{EMA}_{\beta_2}\big((g_t - m_t)^2\big).
\end{equation}
The second axis is the \emph{ordering}. Let $r_t \in \{v_t, s_t\}$ denote the chosen normalizer. \emph{Momentum-then-normalize} methods aggregate gradients before scaling, while \emph{normalize-then-momentum} methods scale the gradient before aggregation. Ignoring bias corrections, the two update directions are:
\begin{equation}\label{pre_post_updates}
    \Delta_t^{\text{pre}} = \frac{\text{EMA}_{\beta_1}(g_t)}{\sqrt{r_t} + \varepsilon} \quad \text{vs} \quad \Delta_t^{\text{post}} = \text{EMA}_{\beta_1}\left( \frac{g_t}{\sqrt{r_t} + \varepsilon} \right).
\end{equation}
In the first case ($\Delta_t^{\text{pre}}$), the EMA is applied to the raw gradient $g_t$. In the second case ($\Delta_t^{\text{post}}$), the EMA acts on the normalized gradient, serving as the final operation.
These choices define a complete $2\times2$ design space. The \emph{momentum-then-normalize} form ($\Delta_t^{\text{pre}}$) includes Adam with second-moment scaling ($r_t=v_t$) and AdaBelief with variance scaling ($r_t=s_t$). The \emph{normalize-then-momentum} form ($\Delta_t^{\text{post}}$) includes LaProp (second-moment scaling) and MVN-Grad (variance normalization).

\textbf{Why the two axes matter.}
MVN-Grad differs from Adam along two complementary design axes: normalizer choice and ordering, each addressing a distinct failure mode. We motivate both modifications below and formalize the resulting guarantees in Section~\ref{sec:theory}.

\begin{figure*}[t!]
  \centering

  \begin{subfigure}[t]{0.42\linewidth}
    \centering
    \includegraphics[width=\linewidth]{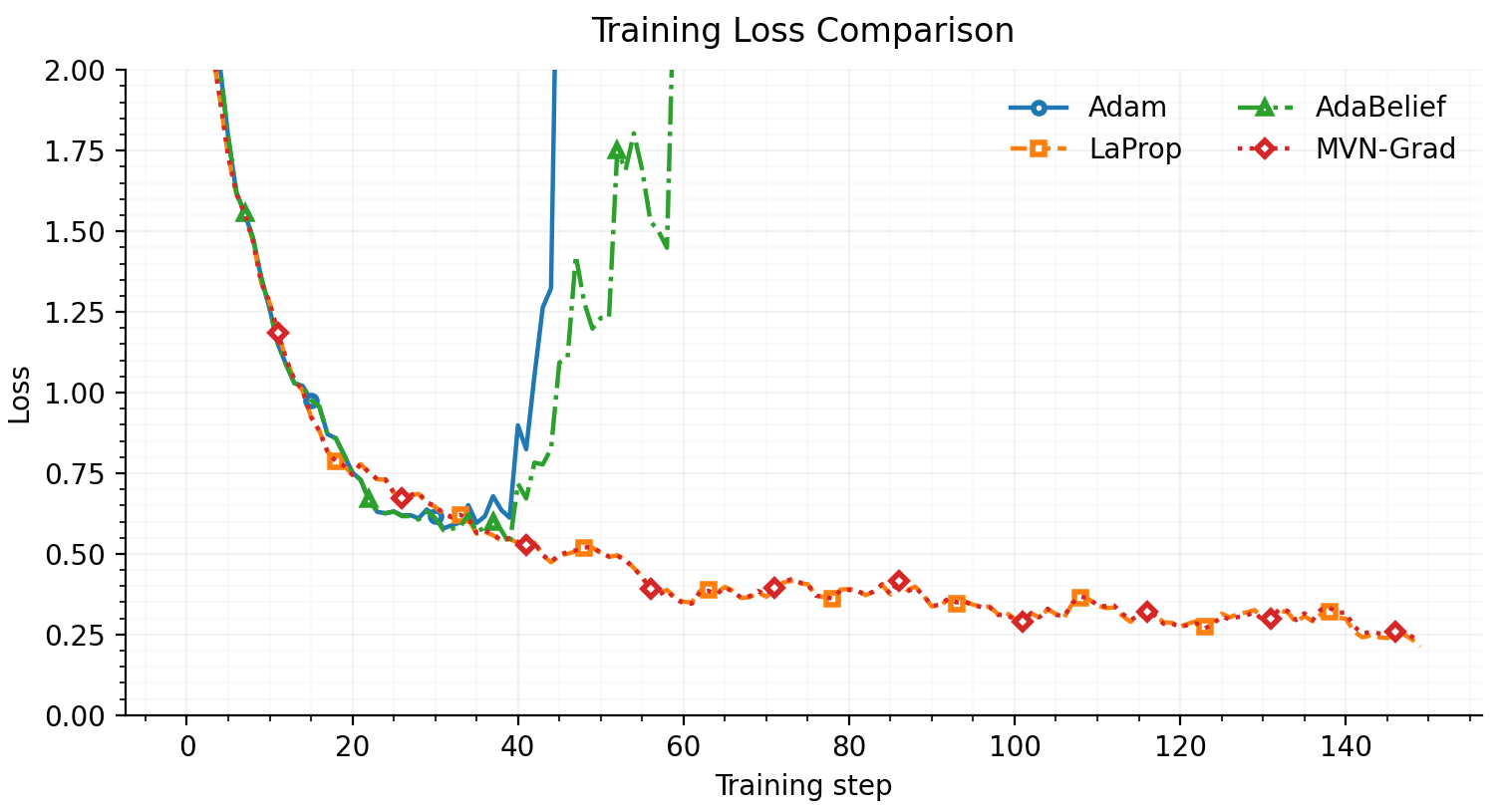}
    \caption{Batch size 128}
    \label{fig:loss_smallbatch}
  \end{subfigure}\quad
  \begin{subfigure}[t]{0.42\linewidth}
    \centering
    \includegraphics[width=\linewidth]{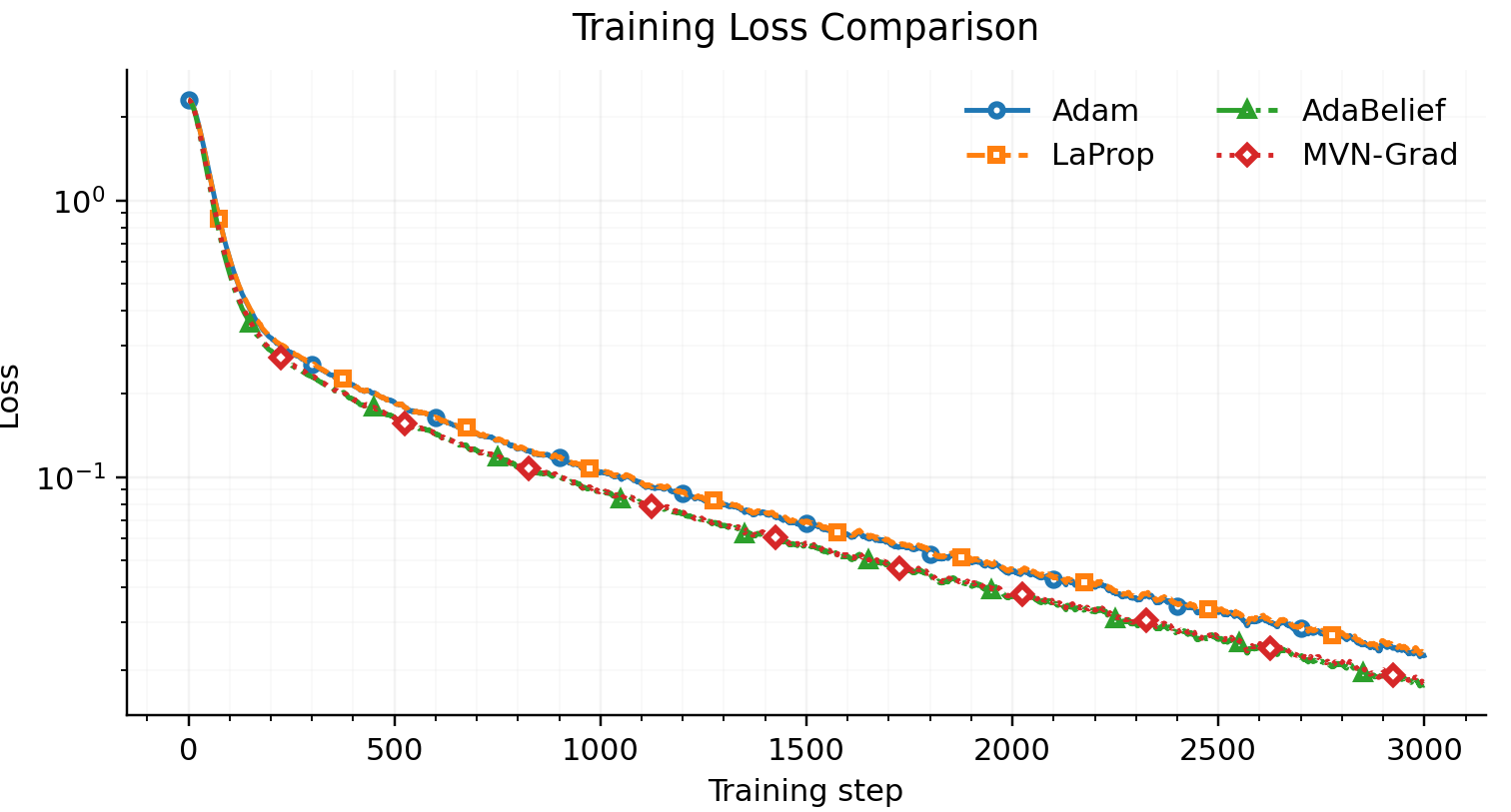}
    \caption{Batch size 1024}
    \label{fig:loss_largebatch}
  \end{subfigure}

  \caption{MNIST training loss (2-layer MLP). {(a)}~At small batch size, MVN-Grad produces smoother dynamics with fewer loss spikes. {(b)}~At large batch size (low-variance), variance normalization preserves gradient scale and yields faster convergence. Hyperparameters in Appendix~\ref{app:mnist-hparams}.}
\end{figure*}

\noindent\textbf{\textit{(A) Choosing the ordering.}} Swapping normalization and momentum changes the \emph{temporal dependency} of the update. Momentum-then-normalize methods (Adam, AdaBelief) divide the momentum buffer $m_t$ by the current normalizer $r_t$. Since $m_t=\beta_1 m_{t-1}+(1-\beta_1)g_t$ contains both past and current gradients, the update can be decomposed as $\Delta_t^{\text{pre}} = \beta_1 m_{t-1}/(\sqrt{r_t}+\varepsilon) + (1-\beta_1)g_t/(\sqrt{r_t}+\varepsilon)$.
The first term divides a history-dependent quantity $m_{t-1}$ by the present-time statistic $r_t$, creating a risky “cross-time” ratio. If the normalizer transiently dips while the momentum buffer still reflects a past spike, this contribution can be artificially amplified. In contrast, normalize-then-momentum forms the ratio $g_t / \sqrt{r_t}$ at the \emph{same timestep}, ensuring the gradient is stabilized \emph{before} entering the momentum buffer and eliminating this coupling; as shown by the toy example in Fig~\ref{fig:loss_smallbatch}.
This behavior is not merely heuristic: the structural change yields two formal stability gains. First, Section~\ref{subsec:cond-moments} shows that applying momentum after normalization reduces the one-step conditional variance by eliminating stochastic interaction between carry-over momentum and a random denominator. Second, this ordering ensures isolated gradient spikes are clipped immediately, formalized in Section~\ref{subsec:spike}. Unlike Adam, where spikes are re-emitted over many steps, MVN-Grad scales outliers at the source. 

\noindent\textbf{\textit{(B) Choosing the normalizer.}}
Switching from the second-moment normalizer $v_t$ to the variance proxy $s_t$ changes the \emph{scale} of the update. The second moment mixes signal and noise and scales with $\mathbb{E}[g_t^2] \approx \mu^2 + \sigma^2$, whereas the variance $s_t$ isolates the noise level $\sigma^2$. This distinction is especially pronounced in low-variance regimes ($\sigma \ll \mu$), which are illustrated by the toy example in Fig.~\ref{fig:loss_largebatch}. In such settings, second-moment normalization effectively rescales the gradient by its own magnitude ($g_t / \sqrt{v_t} \approx g_t / |g_t|$), causing the method to collapse to a conservative, sign-based update. In contrast, variance normalization scales deviations around the mean, preserving informative gradient magnitudes up to a noise-dependent preconditioner. Importantly, this behavior is not merely anecdotal: as we formally prove in Section~\ref{subsec:var-vs-second}, variance-based methods achieve fast convergence in low-noise regions, whereas second-moment methods incur dimension-dependent slowdowns.

\section{Theoretical Guarantees}
\label{sec:theory}

This section justifies MVN-Grad's two design choices. We first show that normalize--then--momentum reduces one-step conditional update variance and limits spike carry-over. We then show that variance normalization avoids sign collapse in low-noise regimes, and finally prove a general bounded-gradient nonconvex convergence guarantee.

Before presenting our results, we fix notation. We analyze coordinate-wise and let $g_t$ denote the stochastic gradient at step $t$ for a fixed coordinate. All random variables are defined on a filtered probability space $(\Omega, \mathcal{F}, (\mathcal{F}_t)_{t\ge0}, \mathbb{P})$, where $\mathcal{F}_{t-1}$ captures the history up to time $t-1$. Let $\mu_t \coloneqq \mathbb{E}[g_t \mid \mathcal{F}_{t-1}]$ denote the conditional mean. For clarity, we omit bias-correction factors of the form $1/(1-\beta^t)$, as they do not affect our conclusions.
We denote the update directions for Adam and AdaBelief by $\Delta_t^{\text{A}}$ and $\Delta_t^{\text{AB}}$, corresponding to the \emph{pre-normalization} structure $\Delta_t^{\text{pre}}$ in \eqref{pre_post_updates} with $r_t=v_t$ and $r_t=s_t$ from \eqref{v_and_s_definitions}, respectively. Likewise, LaProp and MVN-Grad are denoted by $\Delta_t^{\text{LP}}$ and $\Delta_t^{\text{MV}}$, corresponding to the \emph{post-normalization} structure $\Delta_t^{\text{post}}$ with $r_t=v_t$ and $r_t=s_t$, respectively.

\subsection{Conditional variance comparison}
\label{subsec:cond-moments}

We quantify the stability gain of applying momentum \emph{after} normalization while using the same variance proxy $s_t$. Specifically, we compare the one-step \emph{conditional update variance} of AdaBelief and MVN-Grad, $\Var(\Delta_t^{\text{AB}}\mid \mathcal F_{t-1})$ versus $\Var(\Delta_t^{\text{MV}}\mid \mathcal F_{t-1})$. This conditional variance measures the unpredictability of the next update given the past, with smaller values indicating smoother dynamics. Under a standard symmetry assumption on centered gradient noise and an idealized mean-tracking condition on the EMA, the variance gap has a closed form and is always nonnegative.

\begin{theorem}
\label{thm:var_adabelief_mvngrad}
Assume that, conditional on $\mathcal F_{t-1}$, the innovation $\zeta_t:=g_t-m_{t-1}$ is symmetric about zero. Define
$
\Delta\Var_t :=
\Var\!\left[\Delta_t^{\text{AB}}\mid\mathcal F_{t-1}\right]
-
\Var\!\left[\Delta_t^{\text{MV}}\mid\mathcal F_{t-1}\right]$.
Then, we have the following 
$\Delta\Var_t
=
(2\beta_1-\beta_1^2)\, m_{t-1}^2\,
\Var(\frac{1}{\sqrt{s_t}+\varepsilon}\,|\,\mathcal F_{t-1})
\ge 0.$

\end{theorem}

Theorem~\ref{thm:var_adabelief_mvngrad} isolates a concrete stability gap between the two updates.
Even when both use the same variance proxy $s_t$, AdaBelief's update $m_{t-1}/(\sqrt{s_t}+\varepsilon)$ multiplies stale momentum by the current \emph{random} normalizer, injecting an extra variance term proportional to
$m_{t-1}^2 \Var\big((\sqrt{s_t}+\varepsilon)^{-1}\mid\mathcal F_{t-1}\big)$.
MVN-Grad avoids this: its carry-over is $\mathcal F_{t-1}$-measurable and decoupled from the normalizer's randomness, yielding strictly more predictable one-step updates.
The theorem gives an explicit closed-form gap under a local symmetry condition on $g_t-m_{t-1}$; the underlying mechanism persists beyond this setting, as our diagnostics confirm.

\begin{figure}%
    \centering
    \includegraphics[width=0.5\textwidth]{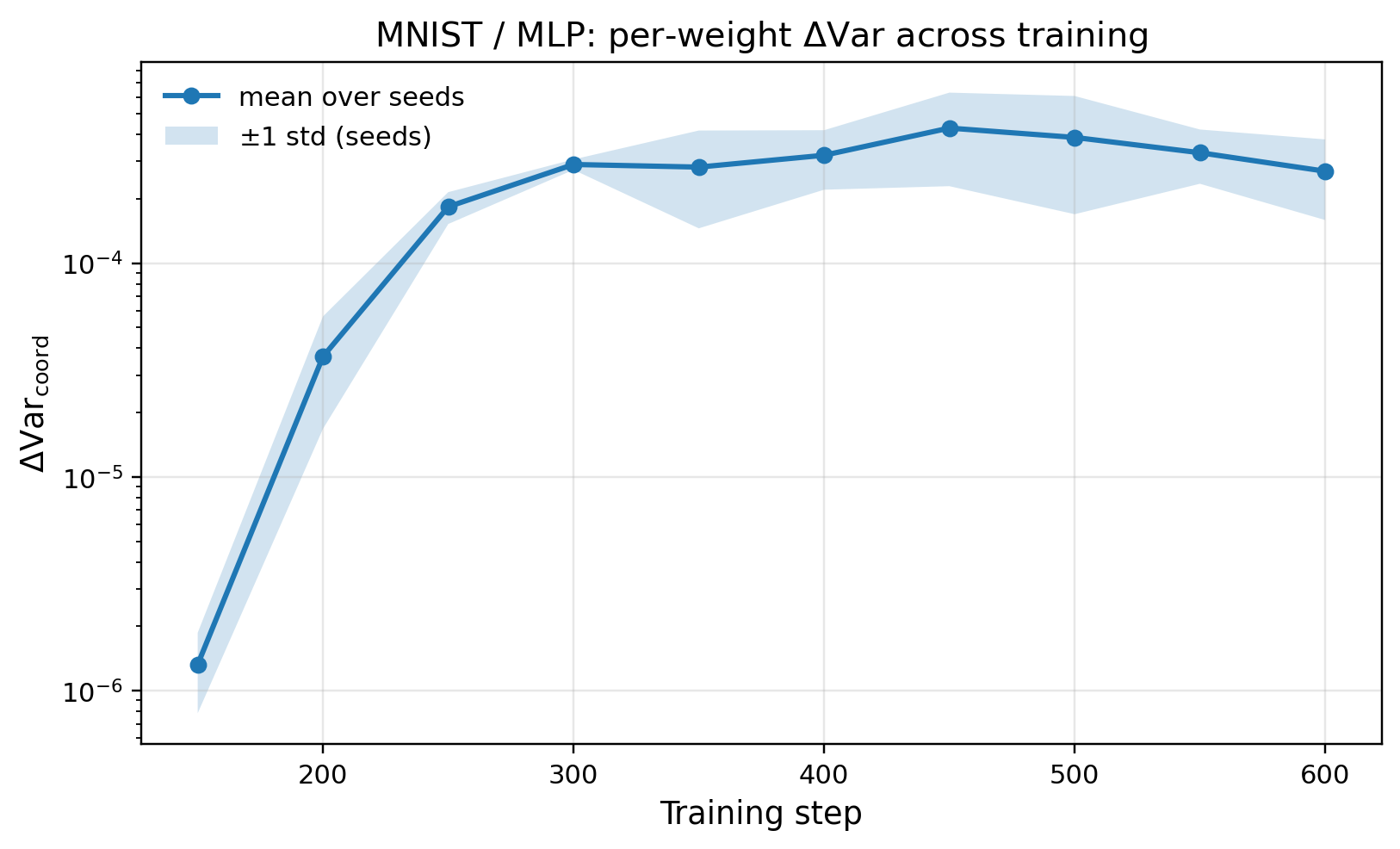}
    \caption{Conditional update-variance gap via Monte Carlo at frozen checkpoints.}
    \label{fig:thm_vgap_ada_prop}
\end{figure}

Fig.~\ref{fig:thm_vgap_ada_prop} estimates the one-step conditional
update-variance gap along a real MNIST training trajectory by freezing optimizer
states and resampling mini-batches. Across checkpoints, the estimated gap is
positive on average, supporting the normalize--then--momentum mechanism beyond
the idealized assumptions of Theorem~\ref{thm:var_adabelief_mvngrad}. The gap
increases during the early EMA transient and then plateaus rather than
vanishing, suggesting that the ordering-induced variance reduction persists
after the optimizer statistics stabilize. Full diagnostic details are in
Appendix~\ref{app:mnist-vgap}.

\begin{proposition}%
\label{prop:approx_mean}
Adopt the setup of Theorem~\ref{thm:var_adabelief_mvngrad}, but replace the symmetry assumption with the following:
conditional on $\mathcal F_{t-1}$, the stochastic gradient admits $g_t = \mu_t + \xi_t$ where 
$\xi_t$ is symmetric about zero,  $\E[\xi_t^2 \mid \mathcal F_{t-1}] \le \sigma_\xi^2$, and the EMA tracking error satisfies $|m_{t-1}-\mu_t|\le \delta$.
Define $Y_t := (\sqrt{s_t}+\varepsilon)^{-1}$. Then:
$\Delta\Var_t
\;\ge\;
(2\beta_1 - \beta_1^2)\,m_{t-1}^2\,\Var\!\left[Y_t \mid \mathcal F_{t-1}\right]
\;-\;
2\beta_1(1-\beta_1)\,|m_{t-1}|\,\big[\delta\,\Var[Y_t \mid \mathcal F_{t-1}] + \sigma_\xi/\varepsilon^2\big].$
In particular, the gap remains non-negative whenever
$|m_{t-1}|$ is sufficiently large relative to $\delta$ and $\sigma_\xi/\varepsilon^2$, i.e.,
$(2-\beta_1)|m_{t-1}|\,\Var[Y_t] \ge 2(1-\beta_1)[\delta\,\Var[Y_t] + \sigma_\xi/\varepsilon^2]$.
\end{proposition}

\subsection{Spike response and momentum carry-over}
\label{subsec:spike}

\begin{figure*}[t!]
  \centering
  \begin{subfigure}{0.45\linewidth}
    \centering
    \includegraphics[width=\linewidth]{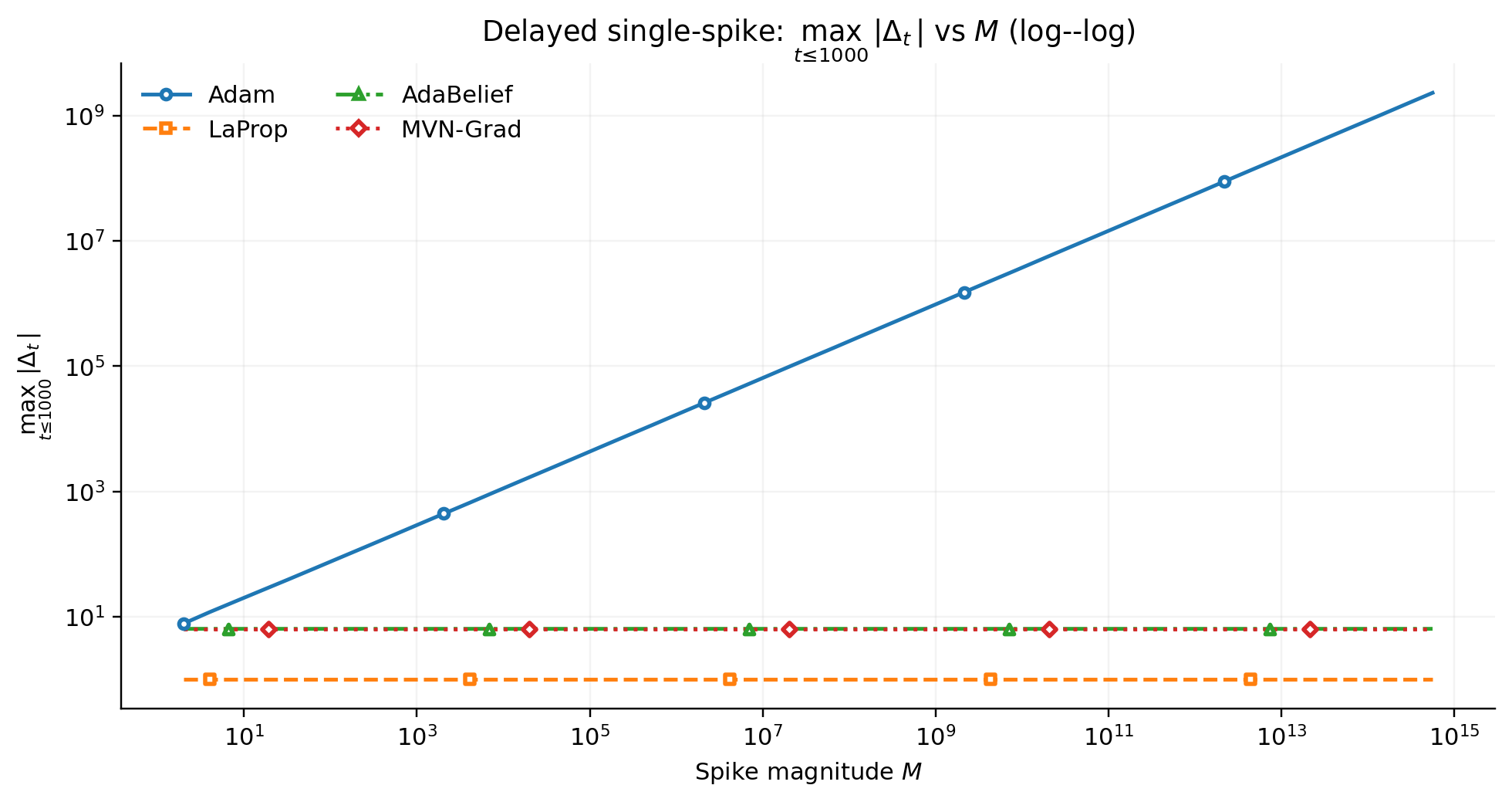}
    \caption{$T=1000$;\; $\beta_1=0.9,\ \beta_2=0.6$.}
    \label{fig:thm_ssr_1_1}
  \end{subfigure}\hfill
  \begin{subfigure}{0.45\linewidth}
    \centering
    \includegraphics[width=\linewidth]{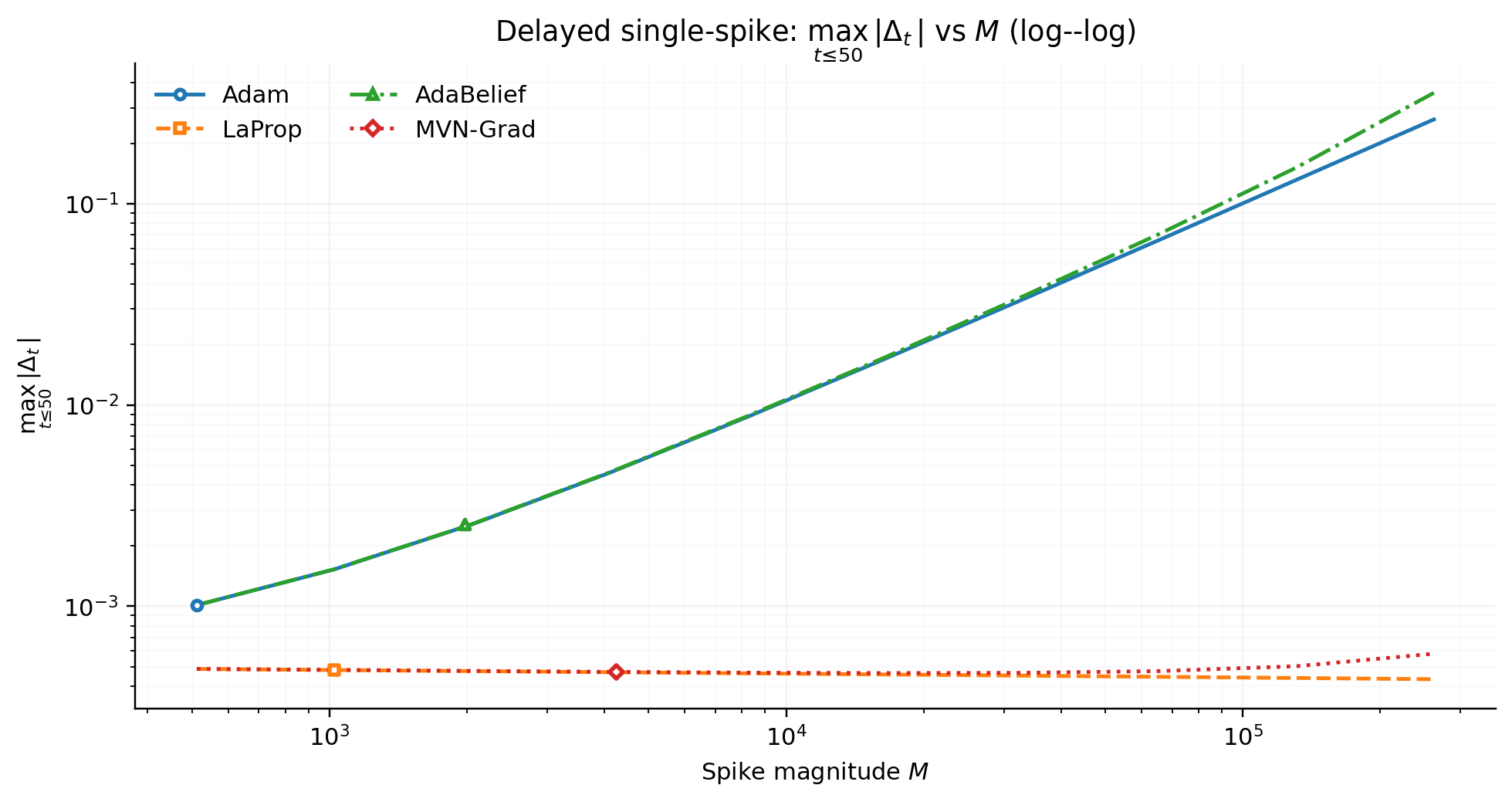}
    \caption{$T=50$;\; $\beta_1=0.99999,\ \beta_2=0.1$.}
    \label{fig:thm_ssr_1_2}
  \end{subfigure}
  \caption{Delayed single-spike robustness: peak update magnitude $\max_{0\le \tau \le T}|\Delta_\tau|$ versus spike size~$M$. Within each panel, all optimizers use the same hyperparameters. Details in Appendix~\ref{app:ssr-hparams}.}
  \label{fig:thm_ssr_1}
\end{figure*}

To make the ordering effect in Section~\ref{subsec:cond-moments} concrete, we study a spike-response model that isolates whether a large gradient is normalized before entering the momentum buffer or stored first and attenuated only later. We begin with a single spike, where the mechanism admits a clean closed-form analysis, and then give an extension to later spikes that dominate the current mean estimate. The purpose of this section is not to claim universal robustness to arbitrary heavy-tailed sequences, but to isolate a concrete carry-over mechanism induced by the ordering of momentum and normalization.

\begin{theorem}\label{thm:spike}
Consider the single-spike model $g_0 = Mu$ for $M\gg 1$ and $u\neq 0$, and $g_t=u$ for all $t\ge 1$.
Assume $v_{-1}=s_{-1}=\bar d>0$ and $\varepsilon>0$.
For LaProp and MVN-Grad, the update magnitudes are uniformly bounded: $|\Delta_t^{\mathrm{LP}}|\le C,\ |\Delta_t^{\mathrm{MV}}|\le C,\ \forall\, t\ge 0,$ where $C$ depends only on $(\beta_1,\beta_2,\bar d,u,\varepsilon)$ and not on $M$. In contrast, for Adam if we define
 $t^\star := \min\Big\{t\ge 1:\ (1-\beta_2)\beta_2^t M^2 \le 1\Big\}$, then there exists $c>0$, depending only on $(\beta_1,\beta_2,\bar d,u,\varepsilon)$, such that
$|\Delta_{t^\star}^{\mathrm{A}}|
\ge c\,(1-\beta_1)\beta_1^{t^\star}M.$

\end{theorem}

Theorem~\ref{thm:spike} shows a contrast in how a spike propagates through the optimizer state.
For Adam, the spike enters the raw momentum $m_0$ and decays as $\beta_1^t M$, while $v_t$ forgets it as $\beta_2^t M^2$.
At time $t^\star$ when the spike contribution in $v_t$ returns to baseline, Adam retains a carry-over of order $(1-\beta_1)\beta_1^{t^\star}M$ in the numerator.
In contrast, LaProp and MVN-Grad normalize the spike at $t=0$ before it enters the momentum buffer, so the resulting updates cannot reintroduce a factor growing with $M$.

Fig.~\ref{fig:thm_ssr_1} corroborates Theorem~\ref{thm:spike}: LaProp and MVN-Grad remain bounded as $M$ grows, while Adam exhibits spike-dependent growth when the denominator forgets the spike faster than the numerator. AdaBelief can amplify spikes in certain regimes due to its variance--momentum coupling. These are mechanism stress tests, not universal robustness claims; details in Appendix~\ref{app:ssr-hparams}.

\begin{remark}\label{rem:adam_bound}
When $\beta_2>\beta_1^2$, Adam's update is uniformly bounded: $|\Delta_t^{\mathrm{A}}| \le \tfrac{1-\beta_1}{\sqrt{1-\beta_2}}\cdot\tfrac{1}{\sqrt{1-\beta_1^2/\beta_2}}$ for all $t\geq 0$. In this regime, the forgetting time $t^\star$ grows logarithmically with $M$, so $\beta_1^{t^\star}$ decays and the lower bound in Theorem~\ref{thm:spike} becomes loose. Outside this regime (e.g., $\beta_1\approx 1$ with smaller $\beta_2$), the momentum remnant can yield spike-dependent growth over practical $M$, as shown in Fig.~\ref{fig:thm_ssr_1}.
\end{remark}

\begin{corollary}%
\label{cor:multi_spike}
For LaProp, the uniform boundedness in Theorem~\ref{thm:spike} extends to arbitrary gradient sequences: since $v_t = \beta_2 v_{t-1} + (1-\beta_2)g_t^2 \ge (1-\beta_2)g_t^2$, any spike $g_t = M$ yields a normalized input $|z_t| \le M/(\sqrt{(1-\beta_2)M^2}+\varepsilon) \le 1/\sqrt{1-\beta_2}$, uniformly in $M$.
For MVN-Grad, the variance proxy uses $s_t = \beta_2 s_{t-1} + (1-\beta_2)\beta_1^2(g_t - m_{t-1})^2$, so
$|z_t| = |g_t|/(\sqrt{s_t}+\varepsilon) \le |g_t|/(\beta_1\sqrt{1-\beta_2}\,|g_t - m_{t-1}|+\varepsilon)$.
For any spike satisfying $|g_t| \ge 2|m_{t-1}|$, we have $|g_t - m_{t-1}| \ge |g_t|/2$, giving $|z_t| \le 2/(\beta_1\sqrt{1-\beta_2})$.
The momentum buffer $u_t$ is a convex combination of the normalized inputs, any such spike contributes a bounded amount to the update. Thus, for MVN-Grad, spike responses remain uniformly bounded for any sequence of spikes satisfying $|g_t|\ge 2|m_{t-1}|$ at the spike times.
\end{corollary}

\subsection{Second moment vs.\ variance normalization}
\label{subsec:var-vs-second}

We now consider a complementary low-variance regime common late in training, where gradients align closely with the signal. Here, the second moment mixes signal and noise ($\mathbb E[g_t^2] = (\mathbb E[g_t])^2 + \Var(g_t)$), whereas variance isolates the noise. When $\Var(g_t) \ll (\mathbb E[g_t])^2$, second-moment normalization becomes overly conservative and collapses to sign-based updates. Conversely, variance normalization preserves magnitude information up to an inverse-noise preconditioner.

To formalize this, we adopt a standard approximation for exponential moving averages. Assuming the gradient distribution varies slowly over the effective window $(1-\beta_2)^{-1}$, we decompose the stochastic gradient into a deterministic signal and a slowly varying noise term. Under this locally stationary approximation, bias-corrected EMAs estimate the gradient's second moment and variance. This requires only moment stability over the EMA window, not global stationarity.

\begin{assumption}
\label{ass:stationary}
At iteration $t$, the stochastic gradient admits the decomposition $
g_t = \nabla F(x_{t-1}) + \xi_t$,
where $\mathbb E[\xi_t \mid \mathcal F_{t-1}] = 0$ and
$\Var(\xi_{t,i} \mid \mathcal F_{t-1}) \approx \sigma_i^2$ remains approximately constant over the EMA window.
The estimators satisfy 
$\mathbb E[ v_{t,i}] \approx \nabla_i F(x_{t-1})^2 + \sigma_i^2$ and $
\mathbb E[ s_{t,i}] \approx \sigma_i^2$.

\end{assumption}

Under Assumption~\ref{ass:stationary}, both LaProp and MVN-Grad admit the coordinate-wise recursion 
$
x_t = x_{t-1} - \eta\,u_t$, where  $u_t = \beta_1 u_{t-1} + (1-\beta_1) z_t$,
but they differ in the normalized direction $z_t$, i.e., 
$z_t^{\mathrm{LP}} = \mathrm{sign}(g_t)\,\sqrt{1\big/(1+\tfrac{v_t-g_t^2}{g_t^2})}$ and $z_t^{\mathrm{MV}} = g_t/\sqrt{s_t}$.
Thus, the comparison reduces to characterizing these directions in the low-variance regime.
To isolate the effect of the normalizer and suppress stochastic fluctuations, we formalize this regime as follows.

\begin{assumption}[Small variance]
\label{ass:small_variance}
Along the trajectory, $g_{t,i} \approx \nabla_i F(x_{t-1})$ and $\sigma_i^2 \ll \nabla_i F(x_{t-1})^2$.
\end{assumption}

Under Assumptions~\ref{ass:stationary}--\ref{ass:small_variance}, the two normalization schemes admit distinct deterministic limits.
We analyze them in turn.
For LaProp, which uses second-moment normalization, the EMA satisfies
$ v_t \approx \nabla_i F(x_{t-1})^2$.
As a result, the corrective factor obeys
$\sqrt{\frac{1}{1+(v_t-g_t^2)/g_t^2}} \approx 1,$ 

and the normalized direction reduces to $
z_t^{\mathrm{LP}} \approx \mathrm{sign}(g_t) \approx \mathrm{sign}(\nabla_i F(x_{t-1}))$.
Thus, in the low-variance regime, LaProp behaves like a momentum-smoothed sign method and largely discards gradient magnitude information.
In contrast, MVN-Grad uses variance normalization, for which $\hat s_t$ estimates the coordinate-wise variance $\sigma_i^2$.
When this estimate is stable and strictly positive on the EMA time scale, the normalized direction satisfies 
$z_t^{\mathrm{MV}} = \frac{g_t}{\sqrt{ s_t}}
\approx \frac{\nabla_i F(x_{t-1})}{\sigma_i}$.
Consequently, MVN-Grad preserves gradient magnitudes up to an inverse-noise preconditioner and avoids collapse to a sign direction.
Although the above discussion applies broadly, focusing on smooth objectives allows a precise separation and sharp convergence rates in low-variance regimes.

\begin{theorem}
\label{thm:separation_MVN-Grad_LaProp_smallvar}
Let $F$ be $L$-smooth and bounded below on $\mathbb{R}^d$, and suppose Assumptions~\ref{ass:stationary}--\ref{ass:small_variance} hold.
In addition, assume a high-SNR stochastic oracle: $g_t=\nabla F(x_{t-1})+\xi_t$ with $\E[\xi_t\!\mid\!\mathcal F_{t-1}]=0$, and for some $\delta\in(0,1)$ and $\sigma\ge 0$, $|\xi_{t,i}|\le \delta\,|\nabla_i F(x_{t-1})|$ a.s. and $\E[\xi_{t,i}^2\mid\mathcal F_{t-1}]\le \sigma^2$ $\forall i$.

\noindent (i) (MVN-Grad, upper bound.) Assume isotropic noise $\sigma_i=\sigma>0$ for all $i$ and a stable variance estimate
$ s_t \approx \sigma^2\mathbf 1$
along the trajectory, so that the limiting MVN-Grad direction satisfies $z_t^{\mathrm{MV}}\approx g_t/\sigma$.
Define $\alpha:=\eta(1-\beta_1)/\sigma$ and assume $\alpha \le (1-\beta_1^2)/L$.
Then it suffices to take $T=\mathcal{O}(1/\varepsilon^2)$ iterations to ensure $\min_{0\le t\le T}\E\|\nabla F(x_t)\|_2\le \varepsilon$.

\noindent (ii) (LaProp, lower bound.) For every horizon $T$, there exists an $L$-smooth function $F$ (satisfying the same assumptions and oracle model above) and an initialization $x_0=x_0(T)$  such that, in the low-variance phase where
$z_t^{\mathrm{LP}}\approx \sign(g_t)$, the following holds.
For LaProp with step size $\eta_t = \tfrac{c}{L\sqrt{t+1}}$, achieving $
\min_{0 \le t \le T} \|\nabla F(x_t)\|_2 \le \varepsilon $
requires $
T = \Omega\!\left(\tfrac{d}{\varepsilon^2}\right)$ iterations.
\end{theorem}

Theorem~\ref{thm:separation_MVN-Grad_LaProp_smallvar} shows a sharp separation between variance-based and second-moment normalization in low-variance regimes. By comparing MVN-Grad and LaProp, two methods that share the same normalize-then-momentum structure, we isolate the effect of the normalizer itself. In this setting, variance normalization yields dimension-free convergence, whereas second-moment normalization induces sign-like dynamics and dimension-dependent slowdown. A similar separation holds for Adam and AdaBelief, though it is confounded by differences in update ordering.

\subsection{General convergence guarantee}
\label{subsec:general-convergence}

Unlike prior results so far, we next establish MVN-Grad's general nonconvex convergence under standard stochastic conditions, without requiring high-SNR or stabilized-variance regimes.

\begin{theorem}
\label{thm:mvngrad_convergence}
Let $F:\mathbb{R}^d \to \mathbb{R}$ be $L$-smooth with
$F_\star := \inf F > -\infty$. Let
$g_t = \nabla F(x_{t-1}) + \xi_t$ with
$\E[\xi_t \mid \mathcal F_{t-1}] = 0$, $\|g_t\|_\infty \le G$ a.s.,
and $\E[\xi_{t,i}^2 \mid \mathcal F_{t-1}] \le \sigma_\xi^2$.
Consider MVN-Grad with constant learning rate $\eta > 0$,
momentum $\beta_1 \in (0,1)$, EMA parameter $\beta_2 \in [0,1)$,
and $\varepsilon > 0$. Also, assume that $\varepsilon \le G$ and
that the stepsize satisfies
$\eta \le c\frac{(1-\beta_1)^2\varepsilon}{L\beta_1^2}$
for a constant $c>0$. Then
\begin{equation}\label{eq:general_rate}
\vspace{-2mm}
\frac{1}{T}\sum_{t=0}^{T-1}
\E\!\left[\left\|\nabla F(x_t)\right\|^2\right]
\le
\frac{C_1 G\Delta_0}{\eta T}
+
C_1 L\eta\,\frac{dG^3}{\varepsilon^2}
+
\frac{C_2 d(1-\beta_2)G^3\sigma_\xi^3}{\varepsilon^4},
\end{equation}
where $\Delta_0 := F(x_0)-F_\star$, and $C_1,C_2>0$ depend only on
$\beta_1$.
\end{theorem}

The third term in~\eqref{eq:general_rate} is a \emph{bias floor} arising from the correlation between $g_{t,i}$ and the normalizer $w_{t,i} = 1/(\sqrt{s_{t,i}}+\varepsilon)$: since $s_t$ depends on $g_t$ through $(g_t - m_t)^2$, the normalized gradient $z_t = g_t \cdot w_t$ is not conditionally unbiased, and the resulting covariance contributes an irreducible error. This floor scales linearly with $(1-\beta_2)$, reflecting that a faster-adapting denominator (smaller $\beta_2$) amplifies the one-step correlation between numerator and denominator. Setting $\beta_2 = 1-1/T$ makes the bias floor $O(dG^3\sigma_\xi^3/(T\varepsilon^4))$, recovering the standard nonconvex rate $O(G^2\sqrt{dL\Delta_0}\big/(\varepsilon\sqrt{T}))$.

\textbf{Challenges in the analysis.}
The main difficulty in analyzing MVN-Grad is the \emph{correlation between the numerator and denominator} of the normalized gradient $z_t = g_t/(\sqrt{s_t}+\varepsilon)$.
In Adam, $v_t = \beta_2 v_{t-1} + (1-\beta_2)g_t^2$ ensures $\sqrt{v_{t,i}} \ge \sqrt{1-\beta_2}\,|g_{t,i}|$, yielding the uniform bound $\|z_t\|_\infty \le 1/\sqrt{1-\beta_2}$ \emph{independent of gradient magnitudes}.
For MVN-Grad, the variance proxy $s_t = \beta_2 s_{t-1} + (1-\beta_2)(g_t - m_t)^2$ tracks deviations from the mean, so when $g_t \approx m_t$ (high signal-to-noise), $s_t$ can be much smaller than $g_t^2$.
Consequently, $z_t$ is bounded by $G/\varepsilon$, introducing $\varepsilon^{-1}$ into the rate to avoid sign collapse. Moreover, because $s_t$ depends on $g_t$, $z_t$ is biased and $\E[z_t \mid \mathcal{F}_{t-1}]$ cannot factor out the denominator as in standard SGD.
The proof uses a virtual-iterate reduction to remove the momentum recursion and a covariance decomposition to control the $g_t$--$s_t$ dependence. The resulting covariance residual gives the bias floor in~\eqref{eq:general_rate}; details in Appendix~\ref{app:proof_convergence}.

\textbf{Comparison with Adam $\&$ LaProp.} $O(G\sqrt{d}/(\sqrt{(1-\beta_2)T}))$~\citep{defossez2020simple} is the best known nonconvex rate for Adam, with no $\varepsilon$ dependence. After optimizing $\eta$, the leading term in~\eqref{eq:general_rate} is $O_{\beta_1}({G^2\sqrt{dL\Delta_0}}/{(\varepsilon\sqrt T)})$.
This looser bound is the cost of variance normalization: the variance proxy can shrink much smaller than $g_t^2$, so the normalized direction is controlled only via $\sqrt{s_{t,i}}+\varepsilon\ge \varepsilon$. Crucially, this flexibility prevents sign collapse (Theorem~\ref{thm:separation_MVN-Grad_LaProp_smallvar}), linking $\varepsilon$-dependence directly to the algorithm's advantages.
While LaProp shares Adam's $1/\sqrt{1-\beta_2}$ bound, Theorem~\ref{thm:separation_MVN-Grad_LaProp_smallvar}(ii) shows it pays an $\Omega(d/\varepsilon^2)$ lower bound in the stabilized regime, whereas MVN-Grad achieves dimension-free $O(1/\varepsilon^2)$.
Indeed, in the stabilized regime ($s_t = \Theta(\sigma^2)$), the effective step size becomes $\eta/\sigma$, replacing the $\varepsilon$-dependence in~\eqref{eq:general_rate} with the noise scale $\sigma$ and eliminating the $(1-\beta_2)$-dependent bias floor. Under exact stabilization ($s_t = \sigma^2\mathbf{1}$), Theorem~\ref{thm:separation_MVN-Grad_LaProp_smallvar}(i) yields an $O(1/T)$ rate via a Heavy-Ball potential argument, strictly improving upon Adam whenever $\sigma \ll G/\sqrt{1-\beta_2}$.

\vspace{2mm}

\section{Experiments}
\label{sec:experiments}

We compare MVN-Grad with Adam, AdaBelief, and LaProp on CIFAR-100/ResNet-18 and GPT-style language modeling. Non-optimizer hyperparameters are fixed, and optimizer hyperparameters are tuned using a coarse-to-focused sweep. Tables report selected-configuration mean $\pm$ std over seeds. Boxplots aggregate all sweep runs: better-positioned boxes indicate stronger performance across the tuning grid, while narrower boxes and shorter tails indicate lower hyperparameter sensitivity.

\subsection{CIFAR-100 with ResNet-18}
\label{subsec:cifar}

We train ResNet-18 on CIFAR-100 for 200 epochs with cosine learning-rate decay, warmup, weight decay, RandAugment, MixUp, and label smoothing.
We consider batch sizes 128 and 1024 to probe different gradient-noise regimes.
We allocate the same sweep budget per optimizer over learning rate $\eta$ and $(\beta_1,\beta_2)$, and (for variance-normalized methods) $\varepsilon_s$.
We select the best configuration with one validation seed, then re-run it with three seeds and report mean $\pm$ std; full grids are in Appendix~\ref{app:hparams-cifar}.

\begin{table}[t]
\centering
\caption{CIFAR-100 / ResNet-18: accuracy (mean $\pm$ std over 3 seeds) for the selected configuration.}
\label{tab:cifar-bs128}
\begin{minipage}[t]{0.48\textwidth}%
\centering
\small
\textbf{Batch size 128}\vspace{2pt}

\begin{tabular}{lcc}
\toprule
\textbf{Optimizer} & \textbf{Test Acc.} & \textbf{Train Acc.} \\
\midrule
Adam & $77.82 \pm 0.18$ & $75.65 \pm 0.71$ \\
LaProp & $77.72 \pm 0.20$ & $75.59 \pm 0.99$ \\
AdaBelief & $79.93 \pm 0.15$ & $81.22 \pm 0.88$ \\
MVN-Grad & $\mathbf{79.94 \pm 0.17}$ & $\mathbf{81.26 \pm 0.94}$ \\
\bottomrule
\end{tabular}
\end{minipage}
\hfill
\begin{minipage}[t]{0.48\textwidth}
\centering
\small
\textbf{Batch size 1024}\vspace{2pt}

\begin{tabular}{lcc}
\toprule
\textbf{Optimizer} & \textbf{Test Acc.} & \textbf{Train Acc.} \\
\midrule
Adam & $78.26 \pm 0.14$ & $83.15 \pm 1.28$ \\
LaProp & $78.40 \pm 0.14$ & $83.05 \pm 1.34$ \\
AdaBelief & $79.34 \pm 0.27$ & $85.89 \pm 1.17$ \\
MVN-Grad & $\mathbf{79.63 \pm 0.12}$ & $\mathbf{85.89 \pm 1.19}$ \\
\bottomrule
\end{tabular}
\label{tab:cifar-bs1024}
\end{minipage}
\end{table}

\begin{figure}
    \centering
    \includegraphics[width=0.5\textwidth]{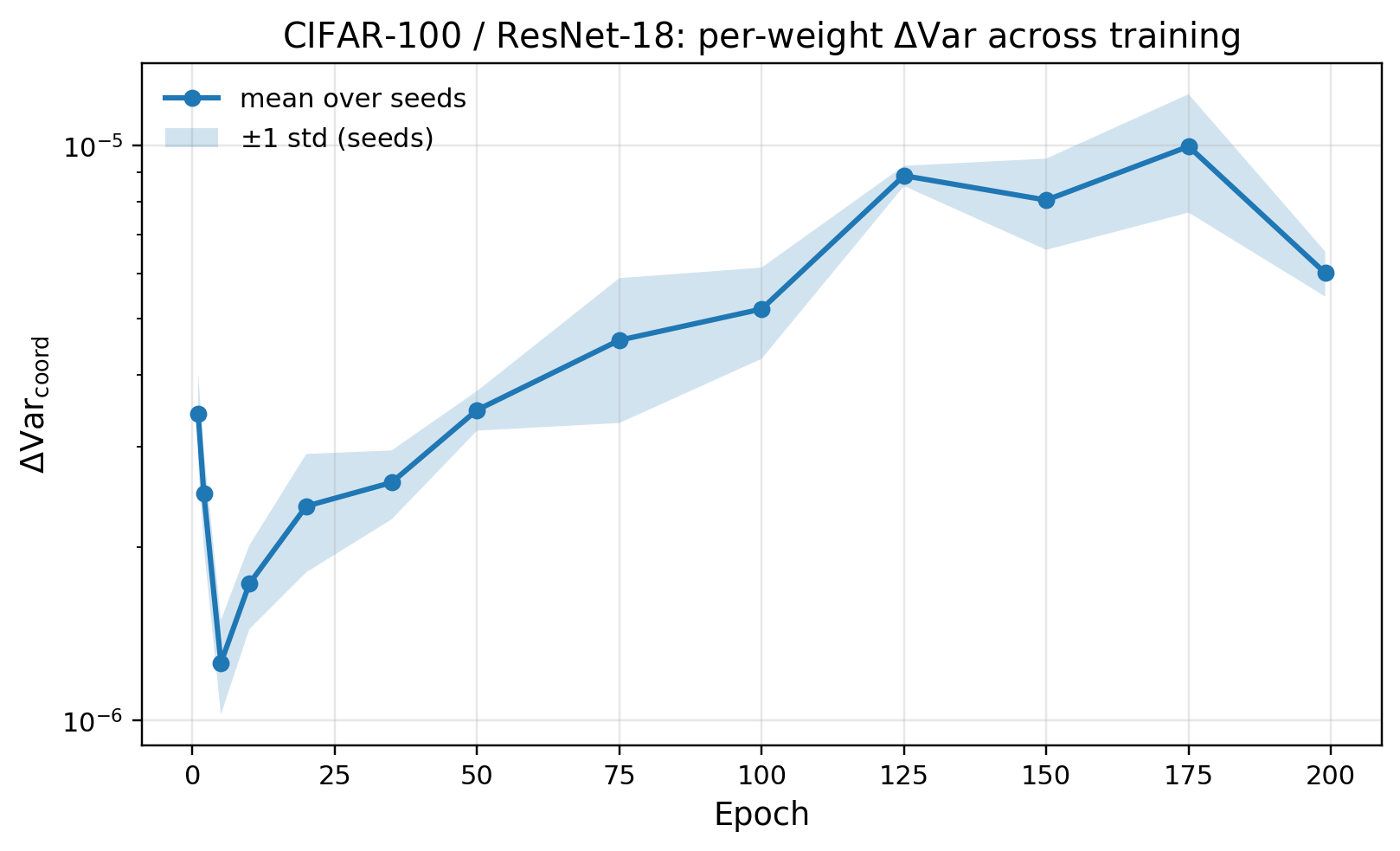}
    \caption{Update-variance gap on CIFAR-100/ResNet-18. The gap remains positive throughout training, confirming Theorem~\ref{thm:var_adabelief_mvngrad}.}
    \label{fig:thm_vgap_ada_prop_cifar}
\end{figure}

Table~\ref{tab:cifar-bs128} reports final accuracies; Figs.~\ref{fig:robust-val-cifar128}, \ref{fig:robust-train-cifar128}, \ref{fig:robust-val-cifar1024}, and \ref{fig:robust-train-cifar1024} show robustness via boxplots over all sweep runs. At batch size 128, MVN-Grad matches AdaBelief and improves over Adam and LaProp by 2.1--2.2\% test accuracy, with comparable sweep medians ($77.72\%$ vs.\ $77.94\%$) and shorter left tails. At batch size 1024, MVN-Grad achieves the best mean test accuracy ($79.63\%$, +0.29\% over AdaBelief) and the strongest sweep robustness: highest median test accuracy ($78.33\%$ vs.\ $76.41\%$ for AdaBelief, $\approx 74\%$ for Adam/LaProp) and highest median train accuracy ($81.65\%$ vs.\ $73.87\%$ for AdaBelief, $\approx 65\%$ for Adam/LaProp). Fig.~\ref{fig:thm_vgap_ada_prop_cifar} validates this on CIFAR-100: the variance gap remains positive across checkpoints, confirming the mechanism of Section~\ref{subsec:cond-moments}.

\begin{figure*}
  \centering
  \begin{subfigure}{0.32\textwidth}
    \centering
    \includegraphics[width=\linewidth]{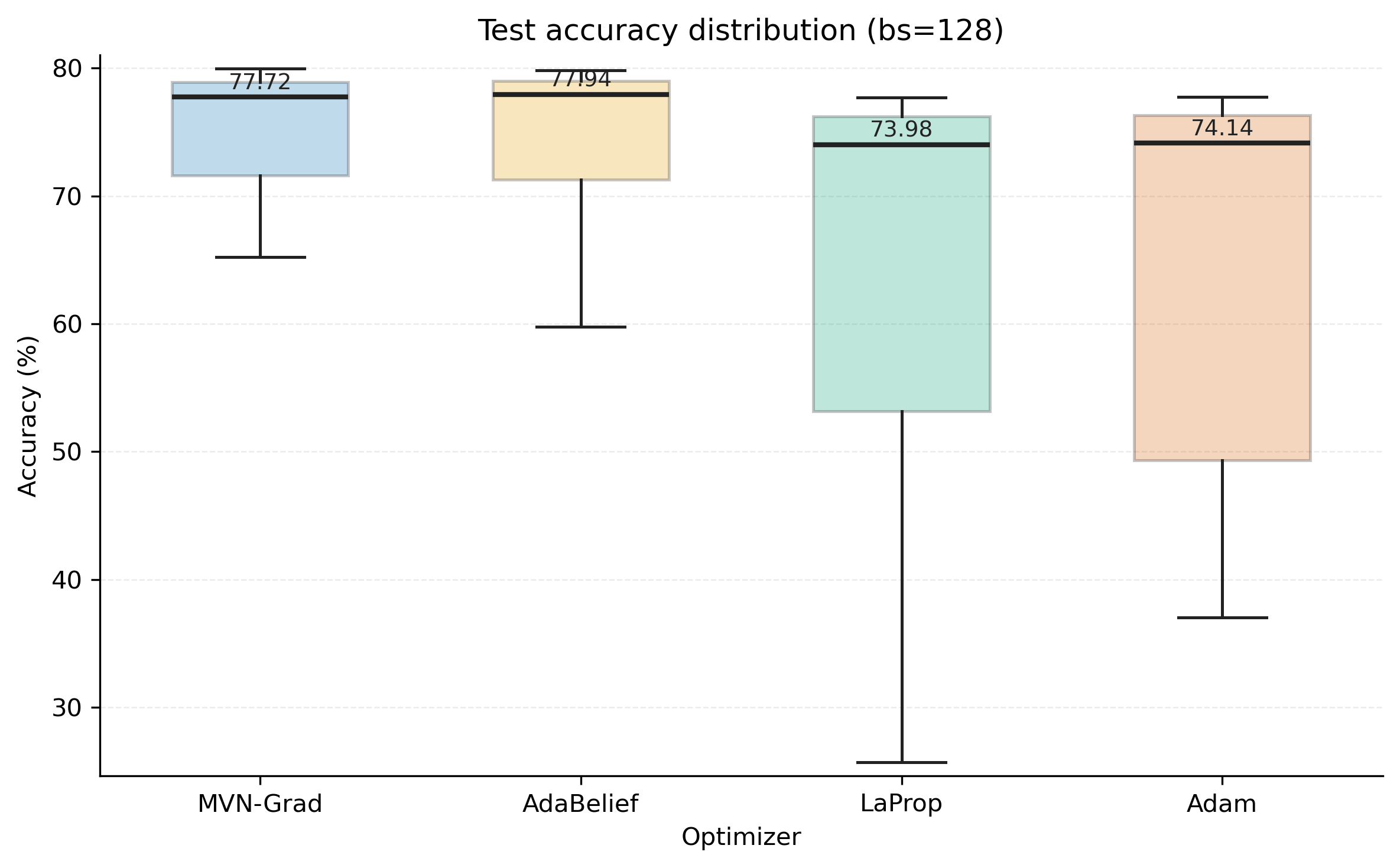}
    \caption{CIFAR-100, bs=128 (test).}
    \label{fig:robust-val-cifar128}
  \end{subfigure}\hfill
  \begin{subfigure}{0.32\textwidth}
    \centering
    \includegraphics[width=\linewidth]{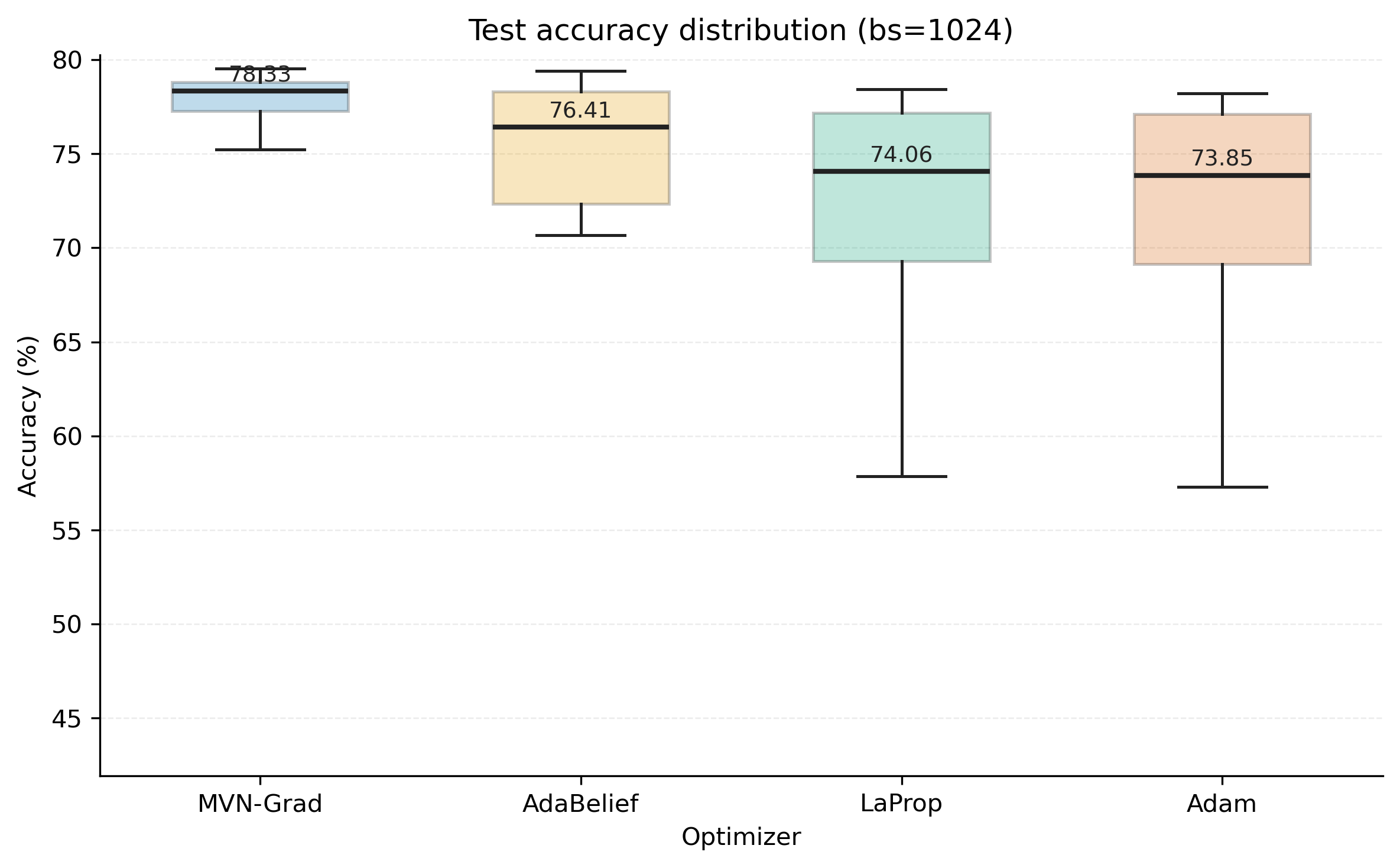}
    \caption{CIFAR-100, bs=1024 (test).}
    \label{fig:robust-val-cifar1024}
  \end{subfigure}\hfill
  \begin{subfigure}{0.32\textwidth}
    \centering
    \includegraphics[width=\linewidth]{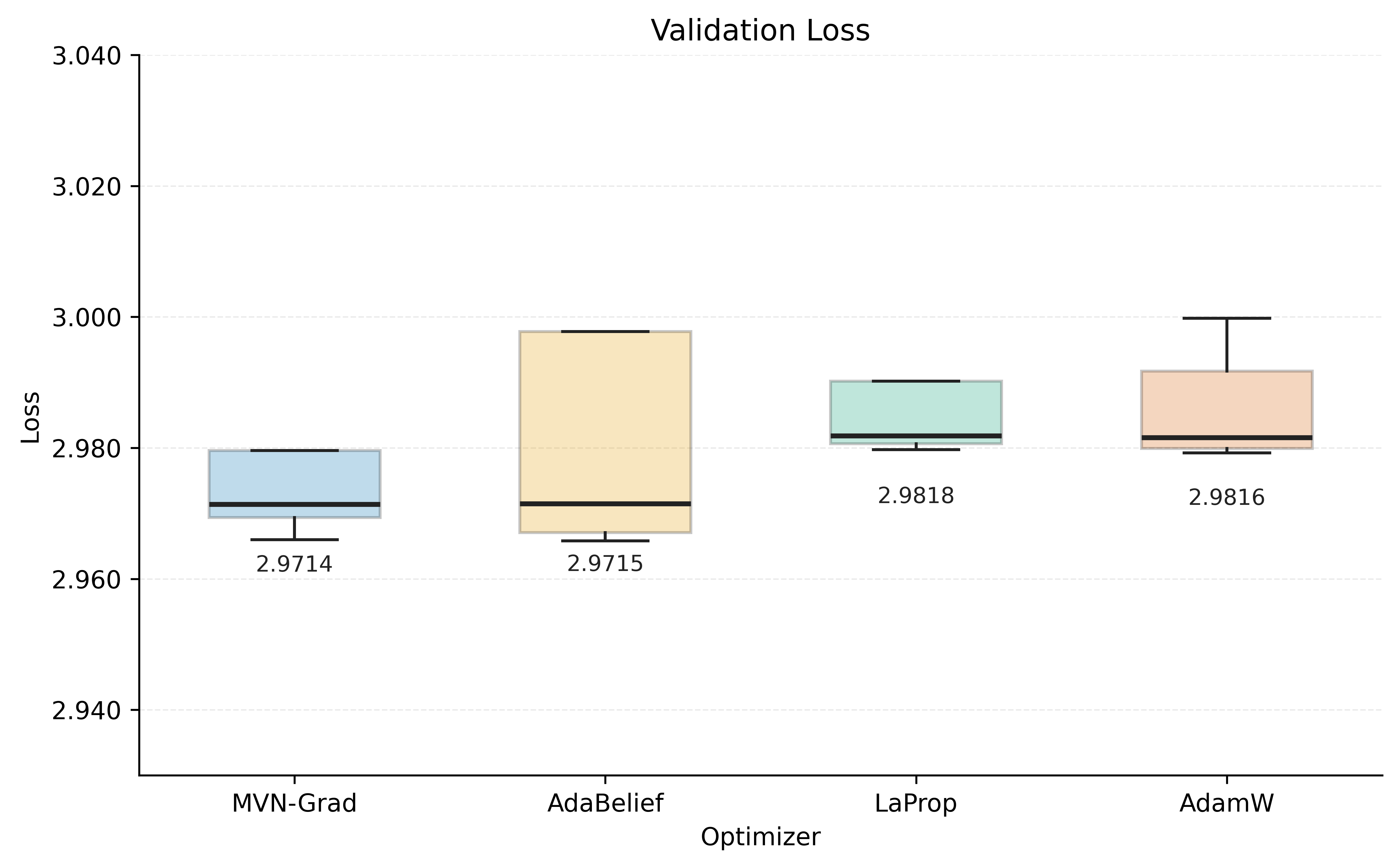}
    \caption{OpenWebText (val loss).}
    \label{fig:robust-val-owt}
  \end{subfigure}\\[4pt]
  \begin{subfigure}{0.32\textwidth}
    \centering
    \includegraphics[width=\linewidth]{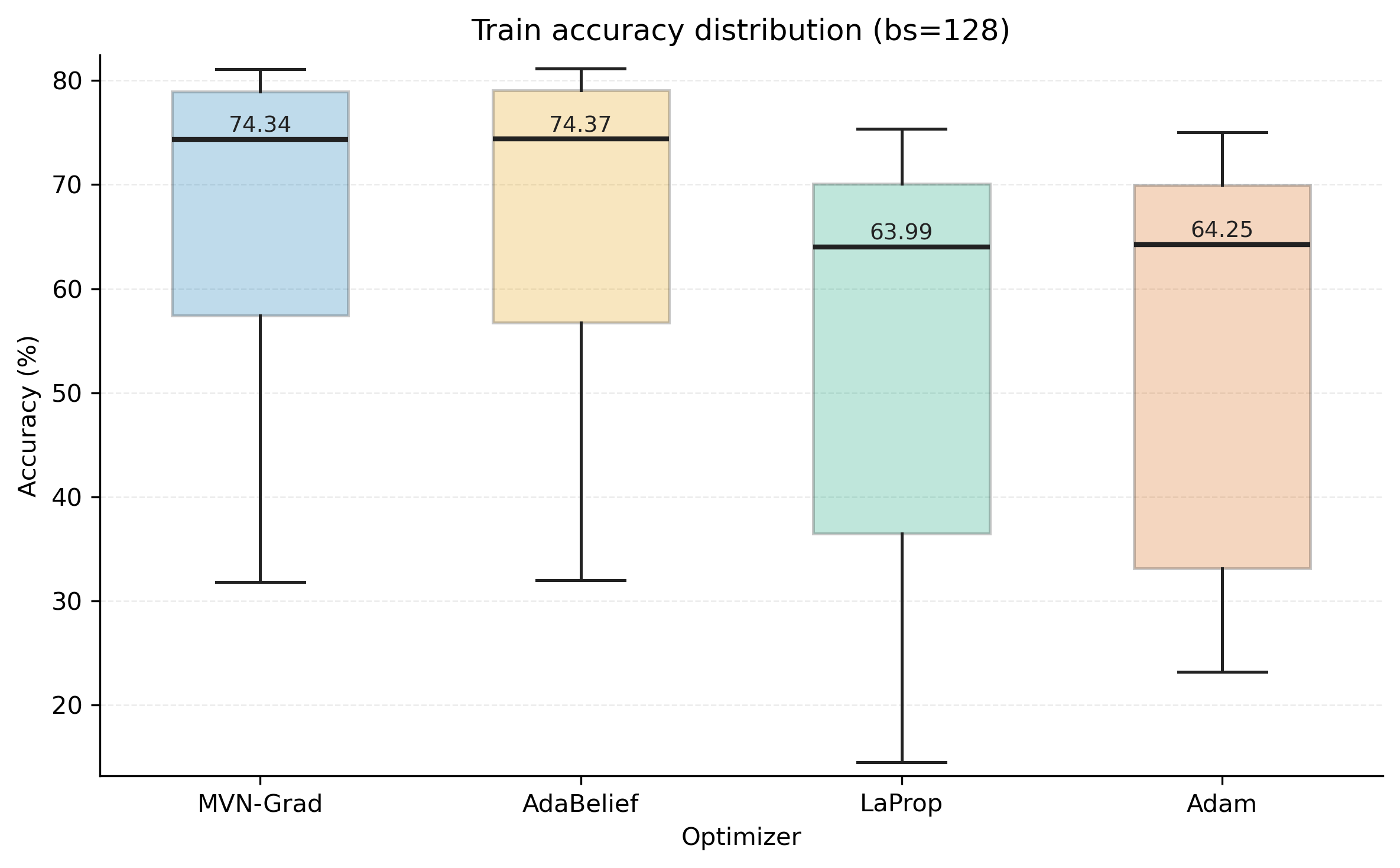}
    \caption{CIFAR-100, bs=128 (train).}
    \label{fig:robust-train-cifar128}
  \end{subfigure}\hfill
  \begin{subfigure}{0.32\textwidth}
    \centering
    \includegraphics[width=\linewidth]{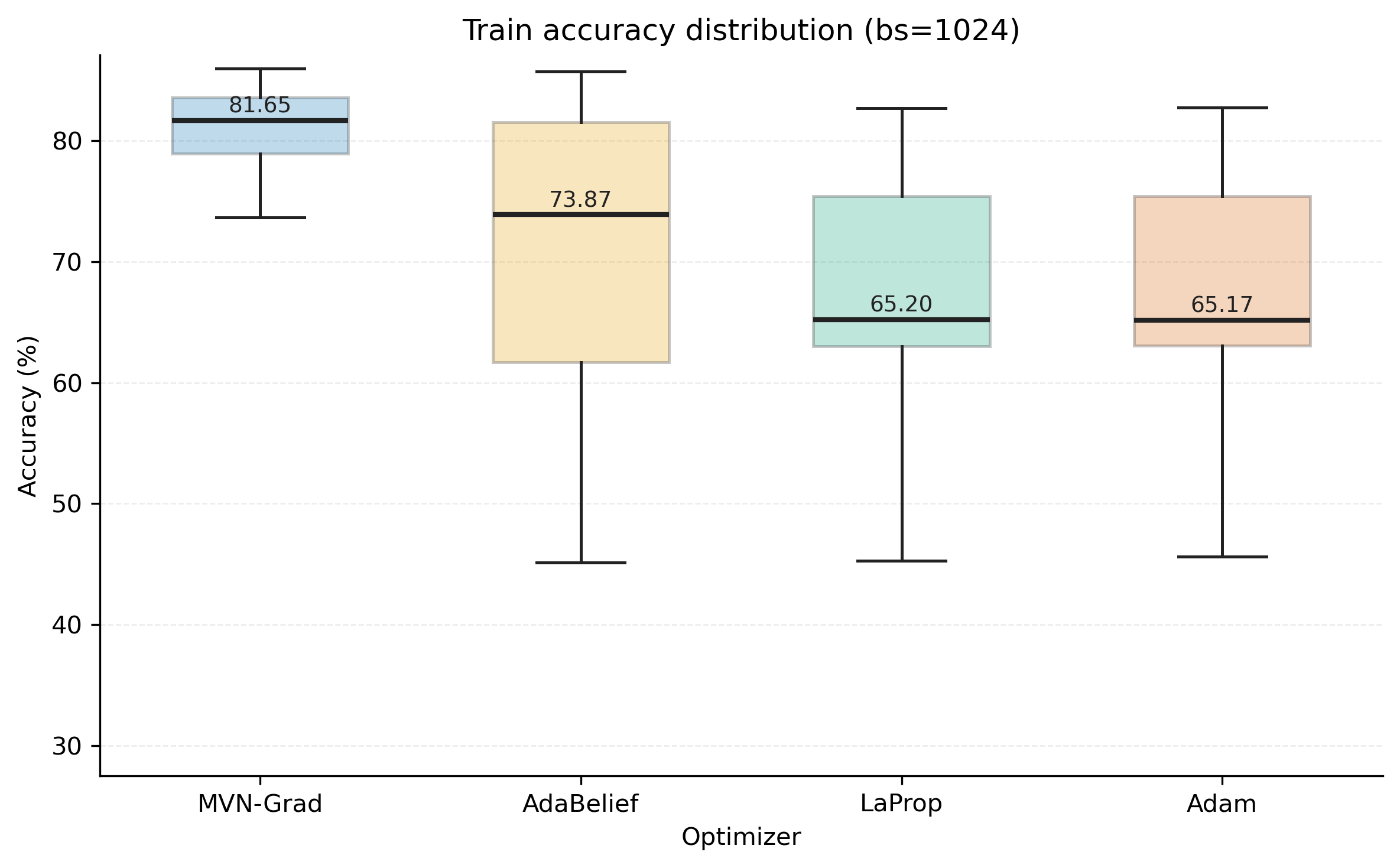}
    \caption{CIFAR-100, bs=1024 (train).}
    \label{fig:robust-train-cifar1024}
  \end{subfigure}\hfill
  \begin{subfigure}{0.32\textwidth}
    \centering
    \includegraphics[width=\linewidth]{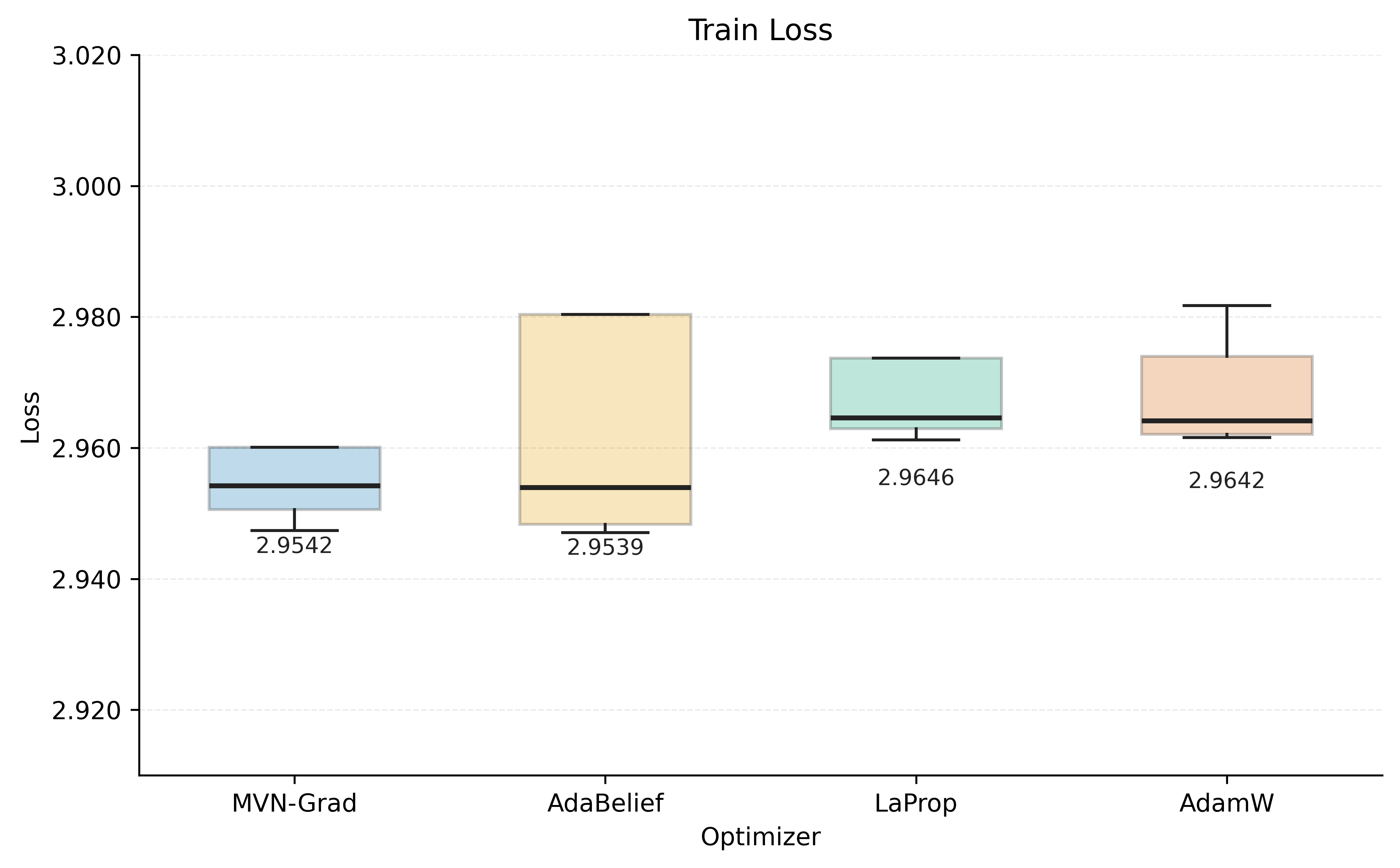}
    \caption{OpenWebText (train loss).}
    \label{fig:robust-train-owt}
  \end{subfigure}
  \caption{Hyperparameter robustness across sweep runs: validation (top) and  training (bottom).}
  \label{fig:robust-val}
  
\end{figure*}

\vspace{-1mm}

\subsection{Language modeling}
\label{subsec:lm}
\vspace{-1mm}

We evaluate two GPT-style language modeling benchmarks: a 30M-parameter model on WikiText-103 and a 124M-parameter model on OpenWebText. We report perplexity $\mathrm{PPL}=\exp(\mathcal{L}_{\text{NLL}})$ on WikiText-103 and validation loss on OpenWebText for consistency with NanoGPT logging.

\textbf{WikiText-103.}
We train a 30M-parameter GPT-2--style Transformer with sequence length 512 on WikiText-103.
All results are reported as mean $\pm$ standard deviation over 4 seeds.
We use a two-stage tuning protocol with shared learning-rate and $(\beta_1,\beta_2)$ grids across optimizers; variance-normalized methods additionally tune $\varepsilon_s$. Full grids are in Appendix~\ref{app:hparams-lm}.
The WikiText-103 experiments use \emph{coupled} $\ell_2$ regularization, in which the term $\lambda x_{t-1}$ is added to the stochastic gradient prior to moment estimation.
Under coupled decay, the regularization signal is filtered through the optimizer’s internal states, and its effect depends on update ordering—specifically, whether momentum is applied before normalization or after normalization.
This interaction can increase seed sensitivity.
For the larger-scale OpenWebText experiments below, we instead use \emph{decoupled} weight decay to enable a cleaner comparison.
As shown in Table~\ref{tab:GPT-Wiki}, MVN-Grad achieves the best mean validation perplexity.
Compared to Adam, MVN-Grad reduces validation perplexity by 1.88, and compared to AdaBelief by 0.93.
We observe higher std for LaProp, and to a lesser extent MVN-Grad, in this setting. %

\begin{table}[t]
\centering
\setlength{\tabcolsep}{3pt}
\caption{Language modeling results (mean $\pm$ std) for the selected configuration.}
\label{tab:GPT-Wiki}
\begin{minipage}[t]{0.48\textwidth}
\centering
\small
\textbf{WikiText-103 (GPT-2 30M)}\vspace{2pt}

\begin{tabular}{lcc}
\toprule
\textbf{Optimizer} & \textbf{Val PPL} & \textbf{Train PPL} \\
\midrule
Adam      & $64.50 \pm 1.16$ & $71.33 \pm 2.55$ \\
LaProp    & $66.96 \pm 5.20$ & $73.63 \pm 5.09$ \\
AdaBelief & $63.55 \pm 1.11$ & $70.92 \pm 1.25$ \\
MVN-Grad  & $\mathbf{62.62 \pm 2.52}$ & $\mathbf{70.01 \pm 3.28}$ \\
\bottomrule
\end{tabular}
\end{minipage}%
\hspace{4pt}%
\begin{minipage}[t]{0.48\textwidth}
\centering
\small
\textbf{OpenWebText (GPT-2 124M)}\vspace{2pt}

\begin{tabular}{lcc}
\toprule
\textbf{Optimizer} & \textbf{Val Loss} & \textbf{Train Loss} \\
\midrule
AdamW        & $2.990 \pm 0.010$ & $2.972 \pm 0.009$ \\
LaPropW      & $2.990 \pm 0.010$ & $2.972 \pm 0.009$ \\
AdaBeliefW   & $2.974 \pm 0.008$ & $2.956 \pm 0.008$ \\
MVN-GradW    & $\mathbf{2.973 \pm 0.007}$ & $\mathbf{2.955 \pm 0.007}$ \\
\bottomrule
\end{tabular}
\label{tab:GPT-Open-loss}
\end{minipage}
\end{table}

\textbf{OpenWebText.}
We train a GPT-2 124M model on OpenWebText with context length 1024, using micro-batches of size 12 with 40 gradient-accumulation steps (effective batch size 480), bf16 precision, gradient clipping at 1, 2000 warmup steps, and 120k training iterations.
We disable dropout and use AdamW-style \emph{decoupled} weight decay with coefficient 0.1.
Based on pilot runs, we fix $\eta=10^{-4}$ and sweep the same $(\beta_1,\beta_2)$ grid for each optimizer; variance-normalized methods use $\varepsilon_s=0$. We select by final validation loss and re-run with three seeds; see Appendix~\ref{app:hparams-lm}.
Table~\ref{tab:GPT-Open-loss} reports final losses for the selected configurations, while Figs.~\ref{fig:robust-val-owt}--\ref{fig:robust-train-owt} visualize robustness via boxplots over all evaluated hyperparameters. MVN-Grad achieves the best mean losses, slightly improving over AdaBeliefW and more clearly outperforming AdamW and LaPropW. Across the full sweep, MVN-GradW attains the lowest median validation loss ($2.9714$), matching AdaBeliefW ($2.9715$) and outperforming AdamW and LaPropW (both $\approx 2.9816$--$2.9818$), with similarly competitive train-loss medians. Moreover, normalize--then--momentum methods show narrower interquartile ranges than momentum--then--normalize methods, indicating less sensitivity to $(\beta_1,\beta_2)$ at a fixed learning rate.

\vspace{-0.5mm}

\noindent\textbf{Memory and runtime overhead.}
MVN-GradW maintains three full-size state tensors per parameter
($m_t$, $s_t$, $u_t$), whereas AdamW, LaPropW, and AdaBeliefW maintain two.
Counting parameters, gradients, and optimizer states, this increases
parameter-sized storage from approximately $4\times$ to $5\times$ model size. In wall-clock time, MVN-GradW remains comparable to
the baselines on OpenWebText: AdamW takes 46.2k\,s, MVN-GradW 46.8k\,s,
LaPropW 47.4k\,s, and AdaBeliefW 48.2k\,s.

\vspace{-0.5mm}

\section{Conclusion}
\vspace{-0.5mm}

We proposed \emph{MVN-Grad}, an Adam-style optimizer that combines variance-based normalization with normalize--then--momentum ordering. Our analysis separated the roles of these two choices: the ordering reduces conditional update variance and bounds spike sensitivity, while variance normalization avoids sign collapse and the resulting dimension-dependent slowdowns in the small-variance (high-signal) regime. On CIFAR-100 and GPT-style language modeling benchmarks, MVN-Grad matched or improved over Adam, AdaBelief, and LaProp, with improved sweep robustness and comparable wall-clock cost.
Our mechanism-isolating theorems rely on local symmetry and high-SNR conditions that clarify structural advantages but do not cover all training regimes; the general convergence rate carries an $\varepsilon^{-1}$ cost from worst-case variance normalization.
Extending the analysis to heavy-tailed noise, larger-scale pretraining, and sharper stabilized-regime rates are directions for future work.

\section*{Acknowledgments}

This work was supported in part by NSF CAREER Award CCF-2338846 and the NSF AI Institute for Foundations of Machine Learning (IFML). We thank the Center for Generative AI (CGAI) and the Texas Advanced Computing Center (TACC) at The University of Texas at Austin for computing support on the Vista GPU Cluster.

\newpage
\appendix
\newpage

\appendix
\section{Extended related work: Adam-style optimizers}
\label{app:related-work}

\paragraph{A taxonomy of Adam variants.}
Most Adam-style optimizers modify (i) the adaptivity statistic and its update rule,
(ii) the coupling between numerator and denominator (including update ordering or temporal decorrelation),
(iii) step-size control (e.g., bounds or variance rectification),
(iv) practical regularization details (e.g., weight decay), and/or
(v) systems constraints such as memory footprint or large-batch training.

\paragraph{Convergence and stability fixes.}
A key theoretical line identifies settings where Adam can fail and proposes convergent variants such as AMSGrad \citep{reddi2019convergence}.
Related stability-oriented modifications include controlling the growth of second-moment accumulators (Yogi \citep{zaheer2018adaptive})
and rectifying early-stage variance of the adaptive learning rate (RAdam \citep{liu2019variance}).

\paragraph{Decoupling momentum and adaptivity.}
Several methods aim to reduce harmful interactions between historical momentum and stochastic normalization.
LaProp \citep{ziyin2020laprop} applies momentum after normalization to separate these effects, which is closely related in spirit to our focus on ordering.
AdaShift \citep{zhou2018adashift} uses temporally shifted statistics to decorrelate numerator and denominator.

\paragraph{Alternative scaling rules.}
Beyond second-moment normalization, AdaBelief \citep{zhuang2020adabelief} adapts step sizes using deviations from the predicted gradient.
This line targets scaling behavior and can be viewed as variance-sensitive normalization.

\paragraph{Practical refinements and generalization-motivated variants.}
AdamW \citep{loshchilov2017fixing} decouples weight decay from the gradient step.
Memory and implementation constraints motivated AdaFactor \citep{shazeer2018adafactor} and SM3 \citep{anil2019memory}.
Other variants aim to mitigate the adaptivity--generalization gap or extreme effective learning rates via bounded/partial adaptivity,
including AdaBound \citep{luo2019adaptive} and Padam \citep{chen2018closing}.
Further modifications address scale-invariant weights (AdamP \citep{heo2020adamp}).

\paragraph{Large-batch and large-scale training.}
Layer-wise adaptations such as LAMB \citep{you2019large} and methods like NovoGrad \citep{ginsburg2019stochastic} target stability and efficiency in large-batch regimes.

\paragraph{Interpretations of Adam and sign-like behavior.}
Several works interpret Adam as combining sign-like directions with variance-dependent magnitudes \citep{balles2018dissecting},
and analyze generalization differences between adaptive methods and SGD \citep{wilson2017marginal}.
These perspectives connect directly to the sign-collapse phenomenon discussed in Section~\ref{subsec:var-vs-second}.

\paragraph{Beyond diagonal second-moment preconditioning.}
Recent optimizers incorporate curvature information using inexpensive diagonal Hessian estimates and clipping,
including Sophia \citep{liu2023sophia} and related diagonal-curvature methods such as AdaHessian \citep{yao2021adahessian}.
We include these for context, though they go beyond strictly Adam-style diagonal second-moment preconditioning.
Related preconditioning methods include Shampoo \citep{pmlr-v80-gupta18a} and Muon \citep{jordan6muon}.

\paragraph{Summary table.}
For quick reference, Table~\ref{tab:app-adam-variants} summarizes representative Adam-style variants, highlighting the core modification each method makes and the primary motivation behind it.

\begin{table*}[ht]
\centering
\small
\setlength{\tabcolsep}{5pt}
\begin{tabular}{p{0.22\textwidth} p{0.33\textwidth} p{0.37\textwidth}}
\toprule
Method(s) & Key change & Main purpose \\
\midrule
AMSGrad \citep{reddi2019convergence} 
& Monotone denominator via $\max$-correction 
& Convergence / avoid pathological steps \\

Yogi \citep{zaheer2018adaptive}, RAdam \citep{liu2019variance}
& Stabilize/rectify second-moment statistics 
& Reduce instability (esp. early / heavy variance) \\

AdaShift \citep{zhou2018adashift}
& Time-shifted statistics 
& Decorrelate numerator and denominator \\

LaProp \citep{ziyin2020laprop}
& Momentum \emph{after} normalization 
& Reduce coupling of stale momentum and noisy scaling \\

AdaBelief \citep{zhuang2020adabelief}
& Deviation/variance-sensitive scaling 
& Improve adaptivity when noise is small \\

AdamW \citep{loshchilov2017fixing}
& Decoupled weight decay 
& Correct regularization under adaptivity \\

AdaFactor \citep{shazeer2018adafactor}, SM3 \citep{anil2019memory}
& Memory-efficient accumulators 
& Scale to large models / reduced memory \\

AdaBound \citep{luo2019adaptive}, Padam \citep{chen2018closing}
& Bounded/partial adaptivity 
& Control extreme per-coordinate LRs; SGD-like late stage \\

AdamP \citep{heo2020adamp}
& Projection for scale-invariant parameters 
& Improve behavior with norm/scale invariances \\

LAMB \citep{you2019large}, NovoGrad \citep{ginsburg2019stochastic}
& Layer-wise scaling / trust ratio 
& Large-batch stability and efficiency \\

Sophia \citep{liu2023sophia}, AdaHessian \citep{yao2021adahessian}
& Diagonal curvature information 
& Curvature-aware scaling beyond second moments \\

\midrule
\multicolumn{3}{l}{\emph{Beyond Adam-style diagonal preconditioning (context):} Shampoo \citep{pmlr-v80-gupta18a}, Muon \citep{jordan6muon}.} \\
\bottomrule
\end{tabular}
\caption{Representative variants around Adam (not exhaustive).}
\label{tab:app-adam-variants}
\end{table*}

\section{Theory}
\subsection{Proof of Theorem~\ref{thm:var_adabelief_mvngrad}}
\begin{proof}

Fix an iteration $t$ and define
\[
Y_t := \frac{1}{\sqrt{s_t}+\varepsilon},
\qquad
X_t := \frac{g_t}{\sqrt{s_t}+\varepsilon}=g_t\,Y_t.
\]
We can write the updates as
\[
\Delta_t^{AB} 
\;=\; \beta_1 m_{t-1}Y_t + (1-\beta_1)X_t,
\]
and, since MVN-Grad applies momentum \emph{after} normalization,
\[
\Delta_t^{MV} \;=\; \beta_1 u_{t-1} + (1-\beta_1)X_t,
\]
where $u_{t-1}$ is $\mathcal F_{t-1}$-measurable.

\paragraph{Conditional variance gap identity.}
Because $m_{t-1}$ and $u_{t-1}$ are $\mathcal F_{t-1}$-measurable, we have
\begin{align*}
\Var\!\left[\Delta_t^{AB}\mid \mathcal F_{t-1}\right]
&=
\Var\!\left[\beta_1 m_{t-1}Y_t + (1-\beta_1)X_t \mid \mathcal F_{t-1}\right] \\
&=
\beta_1^2 m_{t-1}^2 \Var\!\left[Y_t\mid \mathcal F_{t-1}\right]
+
(1-\beta_1)^2 \Var\!\left[X_t\mid \mathcal F_{t-1}\right] \\
&\quad
+2\beta_1(1-\beta_1)m_{t-1}\Cov\!\left[Y_t,X_t\mid \mathcal F_{t-1}\right],
\end{align*}
and
\[
\Var\!\left[\Delta_t^{MV}\mid \mathcal F_{t-1}\right]
=
\Var\!\left[\beta_1 u_{t-1} + (1-\beta_1)X_t \mid \mathcal F_{t-1}\right]
=
(1-\beta_1)^2 \Var\!\left[X_t\mid \mathcal F_{t-1}\right].
\]
Subtracting yields
\begin{equation}
\label{eq:var_gap_AB_MV}
\Delta\Var_t
=
\beta_1^2 m_{t-1}^2 \Var\!\left[Y_t\mid \mathcal F_{t-1}\right]
+
2\beta_1(1-\beta_1)m_{t-1}\Cov\!\left[Y_t,X_t\mid \mathcal F_{t-1}\right].
\end{equation}

\paragraph{Evaluate the covariance term via symmetry.}
Let $\zeta :=g_t-m_{t-1}$. By assumption, conditional on $\mathcal F_{t-1}$,
the random variable $\zeta$ is symmetric about zero. Moreover, using the recursion $m_t=\beta_1 m_{t-1}+(1-\beta_1)g_t$, we have
\[
g_t-m_t = \beta_1(g_t-m_{t-1})=\beta_1 \zeta.
\]
Since $s_t=\beta_2 s_{t-1}+(1-\beta_2)(g_t-m_t)^2$, it follows that
\[
s_t = \beta_2 s_{t-1} + (1-\beta_2)\beta_1^2 \zeta^2,
\qquad\text{and hence}\qquad
Y_t = \frac{1}{\sqrt{\beta_2 s_{t-1} + (1-\beta_2)\beta_1^2 \zeta^2}+\varepsilon}.
\]
Therefore $Y_t$ is an \emph{even} function of $\zeta$, while $\zeta Y_t$ and $\zeta Y_t^2$ are \emph{odd} functions of $\zeta$.
Also,
\[
X_t = g_t Y_t = (\zeta+m_{t-1})Y_t.
\]
Thus,
\begin{align*}
\Cov\!\left[Y_t,X_t\mid \mathcal F_{t-1}\right]
&=
\Cov\!\left[Y_t,(\zeta+m_{t-1})Y_t\mid \mathcal F_{t-1}\right] \\
&=
\Cov\!\left[Y_t,\zeta Y_t\mid \mathcal F_{t-1}\right]
+
m_{t-1}\Var\!\left[Y_t\mid \mathcal F_{t-1}\right].
\end{align*}
To handle the first term, note that
\[
\E[\zeta Y_t\mid \mathcal F_{t-1}] = 0
\quad\text{and}\quad
\E[\zeta Y_t^2\mid \mathcal F_{t-1}] = 0
\]
because $\zeta Y_t$ and $\zeta Y_t^2$ are odd in the (conditionally) symmetric variable $\zeta$. Hence
\[
\Cov\!\left[Y_t,\zeta Y_t\mid \mathcal F_{t-1}\right]
=
\E[\zeta Y_t^2\mid \mathcal F_{t-1}]
-
\E[Y_t\mid \mathcal F_{t-1}]\,\E[\zeta Y_t\mid \mathcal F_{t-1}]
=0,
\]
and therefore
\[
\Cov\!\left[Y_t,X_t\mid \mathcal F_{t-1}\right]
=
m_{t-1}\Var\!\left[Y_t\mid \mathcal F_{t-1}\right].
\]

\paragraph{Conclude.}
Plugging this into~\eqref{eq:var_gap_AB_MV} gives
\[
\Delta\Var_t
=
\beta_1^2 m_{t-1}^2 \Var\!\left[Y_t\mid \mathcal F_{t-1}\right]
+
2\beta_1(1-\beta_1)m_{t-1}^2 \Var\!\left[Y_t\mid \mathcal F_{t-1}\right]
=
(2\beta_1-\beta_1^2)m_{t-1}^2 \Var\!\left[Y_t\mid \mathcal F_{t-1}\right].
\]
Finally, since $\Var[Y_t\mid \mathcal F_{t-1}]\ge 0$ and $2\beta_1-\beta_1^2=\beta_1(2-\beta_1)\ge 0$
for $\beta_1\in[0,1]$, we obtain $\Delta\Var_t\ge 0$, as claimed.

\end{proof}

\subsection{Proof of Proposition~\ref{prop:approx_mean}}
\begin{proof}
We follow the proof of Theorem~\ref{thm:var_adabelief_mvngrad} but account for the tracking error.
Write $\zeta_t := g_t - m_{t-1} = \xi_t + e$ where $e := \mu_t - m_{t-1}$ with $|e| \le \delta$, and $\xi_t$ is symmetric about zero.
The variance gap identity~\eqref{eq:var_gap_AB_MV} still gives
\[
\Delta\Var_t
=
\beta_1^2 m_{t-1}^2 \Var[Y_t \mid \mathcal F_{t-1}]
+
2\beta_1(1-\beta_1)m_{t-1}\Cov[Y_t, X_t \mid \mathcal F_{t-1}],
\]
where $Y_t = (\sqrt{s_t}+\varepsilon)^{-1}$ and $X_t = g_t Y_t = (\zeta + m_{t-1})Y_t$.

Since $\zeta = \xi_t + e$, the substitution $X_t = (\xi_t + e + m_{t-1})Y_t$ gives
\[
\Cov[Y_t, X_t \mid \mathcal F_{t-1}]
=
\Cov[Y_t, \xi_t Y_t \mid \mathcal F_{t-1}]
+
(e + m_{t-1})\Var[Y_t \mid \mathcal F_{t-1}].
\]

We now bound $m_{t-1}\Cov[Y_t, X_t]$ directly, handling both signs of $m_{t-1}$.
From the decomposition above:
\begin{align*}
m_{t-1}\Cov[Y_t, X_t]
&= m_{t-1}(e + m_{t-1})\Var[Y_t] + m_{t-1}\Cov[Y_t, \xi_t Y_t].
\end{align*}
For the first term, $m_{t-1}(e + m_{t-1}) = m_{t-1}^2 + m_{t-1}e$, and by triangle inequality:
\[
m_{t-1}(e + m_{t-1}) \ge m_{t-1}^2 - |m_{t-1}||e| \ge m_{t-1}^2 - |m_{t-1}|\delta.
\]
This holds regardless of the sign of $m_{t-1}$, since $m_{t-1}e \ge -|m_{t-1}||e|$ always.
Multiplying by $\Var[Y_t] \ge 0$:
\[
m_{t-1}(e + m_{t-1})\Var[Y_t] \ge (m_{t-1}^2 - |m_{t-1}|\delta)\Var[Y_t].
\]

For the second term, we bound $|m_{t-1}\Cov[Y_t, \xi_t Y_t]| \le |m_{t-1}| \cdot |\Cov[Y_t, \xi_t Y_t]|$.
By Cauchy--Schwarz:
\[
|\Cov[Y_t, \xi_t Y_t]|
\le \sqrt{\Var[Y_t] \cdot \Var[\xi_t Y_t]}
\le \sqrt{\E[Y_t^2] \cdot \E[\xi_t^2 Y_t^2]}.
\]
Since $Y_t \le 1/\varepsilon$ and $\E[\xi_t^2] \le \sigma_\xi^2$ (bounded by the gradient bound), we obtain:
\[
|\Cov[Y_t, \xi_t Y_t]| \le \frac{\sigma_\xi}{\varepsilon^2}.
\]

Combining:
\begin{align*}
m_{t-1}\Cov[Y_t, X_t]
&\ge (m_{t-1}^2 - |m_{t-1}|\delta)\Var[Y_t] - |m_{t-1}| \cdot \frac{\sigma_\xi}{\varepsilon^2}.
\end{align*}
Plugging into the variance gap identity:
\begin{align*}
\Delta\Var_t
&\ge \beta_1^2 m_{t-1}^2 \Var[Y_t]
+ 2\beta_1(1-\beta_1)\big[(m_{t-1}^2 - |m_{t-1}|\delta)\Var[Y_t] - |m_{t-1}|\sigma_\xi/\varepsilon^2\big] \\
&= (2\beta_1 - \beta_1^2)m_{t-1}^2 \Var[Y_t]
- 2\beta_1(1-\beta_1)|m_{t-1}|\big[\delta\,\Var[Y_t] + \sigma_\xi/\varepsilon^2\big].
\end{align*}
The non-negativity condition $\Delta\Var_t \ge 0$ follows whenever the first term dominates, i.e., when
\[
(2\beta_1 - \beta_1^2)|m_{t-1}|\Var[Y_t] \ge 2\beta_1(1-\beta_1)\big[\delta\,\Var[Y_t] + \sigma_\xi/\varepsilon^2\big],
\]
which holds provided $\delta$ is sufficiently small relative to $|m_{t-1}|$, as stated.
\end{proof}

\color{black}

\subsection{Proof of Theorem~\ref{thm:spike}}

\begin{proof}
We work coordinate-wise (all operations are element-wise), so it suffices to consider the scalar case.
Let $u\neq 0$ be fixed, $g_0=Mu$ with $M\gg 1$, and $g_t=u$ for all $t\ge 1$.
Assume $\Delta_{-1}=0$ and $m_{-1}=0$, and $v_{-1}=s_{-1}=\bar d>0$.

\paragraph{Part I: LaProp and MVN-Grad are uniformly bounded in $M$.}

\emph{LaProp (momentum after second-moment normalization).}
Define
\[
v_t=\beta_2 v_{t-1}+(1-\beta_2)g_t^2,
\qquad
z_t^{\text{LP}} := \frac{g_t}{\sqrt{v_t}+\varepsilon},
\qquad
\Delta_t^{\text{LP}} := \beta_1 \Delta_{t-1}^{\text{LP}} + (1-\beta_1) z_t^{\text{LP}}.
\]
At $t=0$,
\[
v_0=\beta_2 \bar d+(1-\beta_2)M^2u^2 \;\ge\; (1-\beta_2)M^2u^2,
\]
hence
\[
|z_0^{\text{LP}}|
=
\frac{|Mu|}{\sqrt{v_0}+\varepsilon}
\le
\frac{|Mu|}{\sqrt{(1-\beta_2)M^2u^2}}
=
\frac{1}{\sqrt{1-\beta_2}}.
\]
For $t\ge 1$, since $\sqrt{v_t}+\varepsilon\ge \varepsilon$ and $g_t=u$, we have
\[
|z_t^{\text{LP}}|\le \frac{|u|}{\varepsilon}.
\]
Define
\[
Z_{\text{LP}} := \max\Big\{\frac{1}{\sqrt{1-\beta_2}}, \frac{|u|}{\varepsilon}\Big\}.
\]
Then $|z_t^{\text{LP}}|\le Z_{\text{LP}}$ for all $t\ge 0$, and the recursion gives
\[
|\Delta_t^{\text{LP}}|
\le
\beta_1|\Delta_{t-1}^{\text{LP}}|+(1-\beta_1)Z_{\text{LP}}
\le
\max\{|\Delta_{t-1}^{\text{LP}}|, Z_{\text{LP}}\}.
\]
By induction (using $\Delta_{-1}^{\text{LP}}=0$), $|\Delta_t^{\text{LP}}|\le Z_{\text{LP}}$ for all $t\ge 0$, which is independent of $M$.

\emph{MVN-Grad (momentum after variance-style normalization).}
Let
\[
m_t=\beta_1 m_{t-1}+(1-\beta_1)g_t,
\qquad
s_t=\beta_2 s_{t-1}+(1-\beta_2)(g_t-m_t)^2,
\]
and define
\[
z_t^{\text{MV}} := \frac{g_t}{\sqrt{s_t}+\varepsilon},
\qquad
\Delta_t^{\text{MV}} := \beta_1 \Delta_{t-1}^{\text{MV}} + (1-\beta_1) z_t^{\text{MV}}.
\]
At $t=0$, $m_0=(1-\beta_1)Mu$ so $g_0-m_0=\beta_1 Mu$, hence
\[
s_0=\beta_2 \bar d+(1-\beta_2)\beta_1^2 M^2u^2 \;\ge\; (1-\beta_2)\beta_1^2 M^2u^2,
\]
and therefore
\[
|z_0^{\text{MV}}|
=
\frac{|Mu|}{\sqrt{s_0}+\varepsilon}
\le
\frac{|Mu|}{\sqrt{(1-\beta_2)\beta_1^2 M^2u^2}}
=
\frac{1}{\beta_1\sqrt{1-\beta_2}}.
\]
For $t\ge 1$, again $\sqrt{s_t}+\varepsilon\ge\varepsilon$ and $g_t=u$, giving
\[
|z_t^{\text{MV}}|\le \frac{|u|}{\varepsilon}.
\]
Define
\[
Z_{\text{MV}} := \max\Big\{\frac{1}{\beta_1\sqrt{1-\beta_2}}, \frac{|u|}{\varepsilon}\Big\}.
\]
The same induction argument yields $|\Delta_t^{\text{MV}}|\le Z_{\text{MV}}$ for all $t\ge 0$.
Thus we may take $C:=\max\{Z_{\text{LP}},Z_{\text{MV}}\}$, proving the uniform boundedness for LaProp and MVN-Grad.

\paragraph{Part II: Adam.}
For Adam, with
\[
m_t=\beta_1 m_{t-1}+(1-\beta_1)g_t,
\qquad
v_t=\beta_2 v_{t-1}+(1-\beta_2)g_t^2,
\qquad
\Delta_t^{\mathrm{A}} := \frac{m_t}{\sqrt{v_t}+\varepsilon},
\]
the closed forms (using $m_{-1}=0$, $v_{-1}=\bar d$) are
\[
m_t=(1-\beta_1)\beta_1^tMu + (1-\beta_1^t)u,
\]
\[
v_t=\beta_2^{t+1}\bar d + (1-\beta_2)\beta_2^t M^2u^2 + (1-\beta_2^t)u^2.
\]
Define the index
\[
t^\star := \min\Big\{t\ge 1:\ (1-\beta_2)\beta_2^t M^2u^2 \le u^2\Big\}.
\]
Then $(1-\beta_2)\beta_2^{t^\star}M^2u^2\le u^2$ and hence
\[
v_{t^\star}
\le
\beta_2^{t^\star+1}\bar d + u^2 + (1-\beta_2^{t^\star})u^2
\le
\bar d + 2u^2.
\]
Consequently,
\[
\sqrt{v_{t^\star}}+\varepsilon \le \sqrt{\bar d+2u^2}+\varepsilon =: D,
\]
where $D$ depends only on $(\beta_2,\bar d,u,\varepsilon)$.
On the other hand,
since $M>0$ and $u\neq 0$, both terms in $m_{t^\star}=(1-\beta_1)\beta_1^{t^\star}Mu + (1-\beta_1^{t^\star})u$ share the same sign, so
\[
|m_{t^\star}|
=
(1-\beta_1)\beta_1^{t^\star}M|u| + (1-\beta_1^{t^\star})|u|
\ge
(1-\beta_1)\beta_1^{t^\star}M|u|.
\]

Therefore,
\[
|\Delta_{t^\star}^{\mathrm{A}}|
=
\frac{|m_{t^\star}|}{\sqrt{v_{t^\star}}+\varepsilon}
\ge
\frac{(1-\beta_1)\beta_1^{t^\star}M|u|}{D}.
\]
This shows that Adam retains a spike-dependent carry-over term of order
$(1-\beta_1)\beta_1^{t^\star}M$.
\end{proof}

\subsection{Proof of Theorem~\ref{thm:separation_MVN-Grad_LaProp_smallvar}}\label{proof_thm:separation_MVN-Grad_LaProp_smallvar}

In this subsection we prove Theorem~\ref{thm:separation_MVN-Grad_LaProp_smallvar}. 
Recall that throughout we work under Assumptions~\ref{ass:stationary}--\ref{ass:small_variance} and the following oracle model.

\begin{assumption}\label{ass:oracle_second_variance}
Let $F$ be $L$-smooth and bounded below by $F_\star$ on $\mathbb{R}^d$, and suppose
Assumptions~\ref{ass:stationary}--\ref{ass:small_variance} hold.
In addition, assume a high-SNR stochastic oracle: we observe $g_t=\nabla F(x_{t-1})+\xi_t$ such that
$\E[\xi_t\mid\mathcal F_{t-1}]=0$ and, for some $\delta\in(0,1)$ and $\sigma\ge 0$,
\[
|\xi_{t,i}|\le \delta\,|\nabla_i F(x_{t-1})| \quad \text{a.s.}, 
\qquad
\E[\xi_{t,i}^2\mid\mathcal F_{t-1}]\le \sigma^2,
\]
for all $i\in[d]$ and $t\ge 1$.
\end{assumption}

The proof proceeds in two steps: we first establish the convergence rate for MVN-Grad under the above assumptions, and then prove the corresponding lower bound for LaProp.

\textbf{MVN-Grad:} To prove the MVN-Grad part of Theorem~\ref{thm:separation_MVN-Grad_LaProp_smallvar}, we first establish two auxiliary results.
The first shows that, in the low-variance regime where the variance normalizer has stabilized, MVN-Grad reduces to a Heavy-Ball-type recursion driven by the stochastic oracle.

\begin{lemma}
\label{lem:mvn_to_hb_stoch}
Assume that in a low-variance phase the variance normalizer has stabilized so that
$s_t \approx \sigma^2 \mathbf{1}$
for some $\sigma>0$.
Then MVN-Grad reduces to the recursion
\[
z_t=\frac{g_t}{\sigma},
\qquad
u_t=\beta_1 u_{t-1}+(1-\beta_1)z_t,
\qquad
x_t=x_{t-1}-\eta u_t,
\]
Eliminating $u_{t-1}$ yields, for all $t\ge 1$,
\begin{equation}
\label{eq:hb_recursion_stoch}
x_t=x_{t-1}+\beta_1(x_{t-1}-x_{t-2})-\alpha\, g_t,
\qquad
\alpha:=\eta(1-\beta_1)\frac{1}{\sigma},
\end{equation}
where we take $x_{-1}:=x_0$ (equivalently $v_0:=0$).
\end{lemma}

\begin{proof}
From $x_{t-1}=x_{t-2}-\eta u_{t-1}$ we have $u_{t-1}=-(x_{t-1}-x_{t-2})/\eta$.
Using $u_t=\beta_1 u_{t-1}+(1-\beta_1)\frac{g_t}{\sigma}$ gives
\[
x_t=x_{t-1}-\eta u_t
=x_{t-1}+\beta_1(x_{t-1}-x_{t-2})-\eta(1-\beta_1)\frac{g_t}{\sigma},
\]
which is \eqref{eq:hb_recursion_stoch}.
\end{proof}

Lemma~\ref{lem:mvn_to_hb_stoch} reduces the analysis of MVN-Grad in the stabilized low-variance regime to a Heavy-Ball recursion with an additive stochastic driving term.
The next lemma introduces a suitable potential function and establishes a one-step descent inequality, which we will later telescope to obtain the claimed rate.

\begin{lemma}\label{lem:potential_descent_hsnr}
Define the increments $y_t:=x_t-x_{t-1}$ for $t\ge 0$ with $y_0:=0$, and set
$\alpha := \eta(1-\beta)\frac{1}{\sigma}$.
Assume the stepsize condition
\begin{equation}\label{eq:stepsize_mvn_stoch}
\alpha \le \frac{1-\beta^2}{L}.
\end{equation}
Define the potential
\[
\Phi_t \;:=\; F(x_t)+\frac{\beta^2}{2\alpha}\,\|y_t\|^2.
\]
Then for all $t\ge 1$,
\begin{equation}
\label{eq:one_step_descent}
    \E[\Phi_t\mid\mathcal F_{t-1}]
    \;\le\;
    \Phi_{t-1}
    -\frac{\alpha}{2}\,\|\nabla F(x_{t-1})\|^2
    +\frac{\alpha}{2}\,\E[\|\xi_t\|^2\mid\mathcal F_{t-1}].    
\end{equation}

\end{lemma}

\begin{proof}
By $L$-smoothness,
\begin{equation}\label{eq:smoothness_step_stoch_full}
F(x_t)\le F(x_{t-1})+\langle \nabla F(x_{t-1}),y_t\rangle+\frac{L}{2}\|y_t\|^2.
\end{equation}
From the definition of $y_t$ and \eqref{eq:hb_recursion_stoch},
\[
y_t+\alpha \nabla F(x_{t-1})=\beta y_{t-1}-\alpha \xi_t.
\]
Taking squared norms and expanding both sides gives
\[
\|y_t\|^2+2\alpha\langle \nabla F(x_{t-1}),y_t\rangle+\alpha^2\|\nabla F(x_{t-1})\|^2
=
\beta^2\|y_{t-1}\|^2-2\alpha\beta\langle y_{t-1},\xi_t\rangle+\alpha^2\|\xi_t\|^2.
\]
Rearranging,
\[
2\alpha\langle \nabla F(x_{t-1}),y_t\rangle
=
\beta^2\|y_{t-1}\|^2-\|y_t\|^2-\alpha^2\|\nabla F(x_{t-1})\|^2
+\alpha^2\|\xi_t\|^2
-2\alpha\beta\langle y_{t-1},\xi_t\rangle.
\]
Taking conditional expectation given $\mathcal F_{t-1}$ and using
$\E[\xi_t\mid\mathcal F_{t-1}]=0$ (hence $\E[\langle y_{t-1},\xi_t\rangle\mid\mathcal F_{t-1}]=0$) yields
\begin{equation}\label{eq:inner_identity_stoch_cond_full}
2\alpha\,\E\!\left[\langle \nabla F(x_{t-1}),y_t\rangle\mid\mathcal F_{t-1}\right]
=
\beta^2\|y_{t-1}\|^2
-\E[\|y_t\|^2\mid\mathcal F_{t-1}]
-\alpha^2\|\nabla F(x_{t-1})\|^2
+\alpha^2\E[\|\xi_t\|^2\mid\mathcal F_{t-1}].
\end{equation}
Taking conditional expectation in \eqref{eq:smoothness_step_stoch_full} and substituting
\eqref{eq:inner_identity_stoch_cond_full} gives
\begin{align*}
\E[F(x_t)\mid\mathcal F_{t-1}]
&\le
F(x_{t-1})
+\frac{\beta^2}{2\alpha}\|y_{t-1}\|^2
-\frac{1}{2\alpha}\E[\|y_t\|^2\mid\mathcal F_{t-1}]
-\frac{\alpha}{2}\|\nabla F(x_{t-1})\|^2
+\frac{\alpha}{2}\E[\|\xi_t\|^2\mid\mathcal F_{t-1}] \\
& \quad  +\frac{L}{2}\E[\|y_t\|^2\mid\mathcal F_{t-1}]\\
&=
F(x_{t-1})
+\frac{\beta^2}{2\alpha}\|y_{t-1}\|^2
-\frac{\alpha}{2}\|\nabla F(x_{t-1})\|^2
+\frac{\alpha}{2}\E[\|\xi_t\|^2\mid\mathcal F_{t-1}]\\
& \quad -\left(\frac{1}{2\alpha}-\frac{L}{2}\right)\E[\|y_t\|^2\mid\mathcal F_{t-1}].
\end{align*}
Adding $\frac{\beta^2}{2\alpha}\E[\|y_t\|^2\mid\mathcal F_{t-1}]$ to both sides yields
\begin{align*}
\E\!\left[F(x_t)+\frac{\beta^2}{2\alpha}\|y_t\|^2\ \middle|\ \mathcal F_{t-1}\right]
&\le
F(x_{t-1})+\frac{\beta^2}{2\alpha}\|y_{t-1}\|^2
-\frac{\alpha}{2}\|\nabla F(x_{t-1})\|^2
+\frac{\alpha}{2}\E[\|\xi_t\|^2\mid\mathcal F_{t-1}]\\
&\quad
-\left(\frac{1-\beta^2}{2\alpha}-\frac{L}{2}\right)\E[\|y_t\|^2\mid\mathcal F_{t-1}].
\end{align*}
Under \eqref{eq:stepsize_mvn_stoch}, the last term is nonpositive and can be dropped.
Recognizing the definition of $\Phi_t$ completes the proof.
\end{proof}

Lemma~\ref{lem:potential_descent_hsnr} introduceda one-step descent inequality. Note that by Assumption~\ref{ass:oracle_second_variance}, we have a.s.
$\|\xi_t\|^2=\sum_{i=1}^d \xi_{t,i}^2 \le \delta^2\sum_{i=1}^d (\nabla_iF(x_{t-1}))^2
=\delta^2\|\nabla F(x_{t-1})\|^2$.
Hence \eqref{eq:one_step_descent} reduces to:
\[
\E[\Phi_t\mid\mathcal F_{t-1}]
\le
\Phi_{t-1}
-\frac{\alpha}{2}(1-\delta^2)\|\nabla F(x_{t-1})\|^2.
\]
Taking total expectation gives
\[
\E[\Phi_t]
\le
\E[\Phi_{t-1}]
-\frac{\alpha}{2}(1-\delta^2)\E\|\nabla F(x_{t-1})\|^2.
\]

Summing from $t=1$ to $T$ yields
\[
\frac{\alpha}{2}(1-\delta^2)\sum_{t=1}^T \E\|\nabla F(x_{t-1})\|^2
\le
\E[\Phi_0]-\E[\Phi_T]
\le
\E[\Phi_0]-F_\star.
\]
Since $y_0=0$, $\Phi_0=F(x_0)$. Therefore,
\[
\frac{\alpha}{2}(1-\delta^2)\sum_{t=1}^T \E\|\nabla F(x_{t-1})\|^2
\le
F(x_0)-F_\star.
\]
Dividing by $T$ gives
\[
\min_{0\le t\le T-1}\E\|\nabla F(x_t)\|^2
\le
\frac{1}{T}\sum_{t=0}^{T-1}\E\|\nabla F(x_t)\|^2
\le
\frac{2(F(x_0)-F_\star)}{\alpha(1-\delta^2)\,T},
\]
Finally, Jensen implies
$\min_{0\le t\le T-1}\E\|\nabla F(x_t)\|
\le
\sqrt{\min_{0\le t\le T-1}\E\|\nabla F(x_t)\|^2}$,
so choosing $T$ as stated ensures the $\varepsilon$-stationarity bound.

\textbf{LaProp.}
We now prove the LaProp part of Theorem~\ref{thm:separation_MVN-Grad_LaProp_smallvar}.
In the small-variance regime, LaProp’s second-moment normalization collapses to a (momentum-smoothed) sign method
(see Section~\ref{subsec:var-vs-second}). Accordingly, we analyze the limiting sign dynamics on a smooth quadratic and
show that, under the same high-SNR oracle as Assumption~\ref{ass:oracle_second_variance}, LaProp requires
$\Omega(d/\varepsilon^2)$ iterations to reach an $\varepsilon$-stationary point under the standard nonconvex stepsize scale.

\paragraph{Objective and oracle.}
Throughout this proof we work with the $L$-smooth function
\[
F(x)=\frac{L}{2}\|x\|_2^2,\qquad x\in\mathbb{R}^d,
\]
so that $\nabla F(x)=Lx$ and the unique stationary point is $x^\star=0$.
We use the same stochastic oracle model as in Assumption~\ref{ass:oracle_second_variance}; in particular,
\[
g_t = \nabla F(x_{t-1})+\xi_t = Lx_{t-1}+\xi_t,
\qquad
\E[\xi_t\mid\mathcal F_{t-1}]=0,
\qquad
|\xi_{t,i}|\le \delta\,|\nabla_i F(x_{t-1})|\ \text{a.s.}
\]

\paragraph{Limiting sign-collapse dynamics.}
As discussed in Section~\ref{subsec:var-vs-second}, in the high-SNR small-variance regime LaProp reduces to the
(stochastic) momentum-smoothed sign recursion
\begin{equation}\label{eq:laprop-sign-stoch}
x_t \;=\; x_{t-1}-\eta_{t-1}\,u_{t-1},
\qquad
u_{t-1} \;:=\; (1-\beta)\sum_{k=0}^{t-1}\beta^{t-1-k}\,\sign(g_{k+1}),
\end{equation}
where $\beta\in[0,1)$ and $\{\eta_t\}_{t\ge 0}$ is any positive stepsize sequence.

\begin{lemma}
\label{lem:sign_preserve}
Assume the oracle of Assumption~\ref{ass:oracle_second_variance}. If $x_{t-1}\succ 0$, then
$\sign(g_t)=\mathbf 1$ a.s. Consequently, along any horizon for which $x_s\succ 0$ holds for all $s\le T$,
the recursion \eqref{eq:laprop-sign-stoch} coincides (a.s.) with the deterministic sign-collapse recursion
\begin{equation}\label{eq:laprop-sign-collapse-det}
x_t \;=\; x_{t-1}-\eta_{t-1}\,u_{t-1},
\qquad
u_{t-1} \;=\; (1-\beta)\sum_{k=0}^{t-1}\beta^{t-1-k}\,\mathbf 1 \;=\; (1-\beta^{t})\,\mathbf 1.
\end{equation}
\end{lemma}

\begin{proof}
Fix $t\ge 1$ and coordinate $j$. If $x_{t-1}\succ 0$, then $x_{t-1,j}>0$ and $\nabla_j F(x_{t-1})=Lx_{t-1,j}>0$.
By Assumption~\ref{ass:oracle_second_variance},
\[
g_{t,j}=Lx_{t-1,j}+\xi_{t,j}\ \ge\ (1-\delta)Lx_{t-1,j}\ >\ 0,
\]
hence $\sign(g_{t,j})=+1$ for all $j$, i.e., $\sign(g_t)=\mathbf 1$ a.s.
Substituting $\sign(g_{k+1})=\mathbf 1$ into \eqref{eq:laprop-sign-stoch} yields \eqref{eq:laprop-sign-collapse-det}.
\end{proof}

\begin{lemma}\label{lem:laprop_finite_horizon}
Fix any horizon $T\ge 1$, any stepsizes $\eta_0,\ldots,\eta_{T-1}>0$, and any $\beta\in[0,1)$.
Define
\begin{equation}\label{eq:a-def-stoch}
a \;:=\; \sum_{t=0}^{T-1}\eta_t\,(1-\beta^{t+1}) \;+\; \frac{\eta_{T-1}}{2}\,(1-\beta^{T}),
\qquad
x_0 := a\,\mathbf{1}\in\mathbb{R}^d,
\end{equation}
where $\mathbf{1}=(1,\ldots,1)$.
Consider the recursion \eqref{eq:laprop-sign-stoch} on $F(x)=\tfrac{L}{2}\|x\|^2$ under the oracle of
Assumption~\ref{ass:oracle_second_variance}. Then, for every $t=0,1,\ldots,T$, the following hold almost surely:
\begin{enumerate}
\item \textbf{Positive orthant invariance.} $x_t\succ 0$. In particular, $\sign(g_t)=\mathbf 1$ for all $t\le T$ a.s.
\item \textbf{Closed form.}
\[
u_t = (1-\beta^{t+1})\,\mathbf{1},
\qquad
x_t = \Big(a - \sum_{k=0}^{t-1}\eta_k(1-\beta^{k+1})\Big)\mathbf{1},
\qquad
x_T=\frac{\eta_{T-1}}{2}(1-\beta^{T})\,\mathbf{1}.
\]
\item \textbf{Gradient lower bound.}
\[
\min_{0\le t\le T}\|\nabla F(x_t)\|_2 \;=\; \|\nabla F(x_T)\|_2
\;=\; \frac{L}{2}(1-\beta^{T})\,\eta_{T-1}\,\sqrt{d}.
\]
Equivalently, to guarantee $\min_{0\le t\le T}\|\nabla F(x_t)\|_2\le \varepsilon$, it is necessary that
\begin{equation}\label{eq:eta-necessary-stoch}
\eta_{T-1} \;\le\; \frac{2\varepsilon}{L(1-\beta^{T})\sqrt{d}}.
\end{equation}
\end{enumerate}
\end{lemma}

\begin{proof}
\textbf{Positive orthant invariance.}
By Lemma~\ref{lem:sign_preserve}, as long as $x_k\succ 0$ we have $\sign(g_{k+1})=\mathbf 1$ a.s.
Assume temporarily that $\sign(g_{k+1})=\mathbf{1}$ for all $k\le t$. Then
\[
u_t=(1-\beta)\sum_{k=0}^t \beta^{t-k}\mathbf{1}=(1-\beta^{t+1})\,\mathbf{1},
\]
and \eqref{eq:laprop-sign-stoch} becomes
\begin{equation}\label{eq:unrolled-stoch}
x_{t+1}=x_t-\eta_t(1-\beta^{t+1})\,\mathbf{1}.
\end{equation}

\textbf{Closed form.}
Unrolling \eqref{eq:unrolled-stoch} gives, for any $t\le T$,
\[
x_t
= x_0 - \sum_{k=0}^{t-1}\eta_k(1-\beta^{k+1})\,\mathbf{1}
= \Big(a - \sum_{k=0}^{t-1}\eta_k(1-\beta^{k+1})\Big)\mathbf{1}.
\]
Using \eqref{eq:a-def-stoch}, for every $t\le T$,
\[
a - \sum_{k=0}^{t-1}\eta_k(1-\beta^{k+1})
\;\ge\;
a - \sum_{k=0}^{T-1}\eta_k(1-\beta^{k+1})
=
\frac{\eta_{T-1}}{2}(1-\beta^{T})
>0,
\]
so $x_t\succ 0$ for all $t\le T$. Lemma~\ref{lem:sign_preserve} then validates the constant-sign hypothesis,
and the expression for $x_T$ follows by plugging $t=T$.

\textbf{Gradient lower bound.}
Since $\nabla F(x)=Lx$,
\[
\|\nabla F(x_T)\|_2 = L\|x_T\|_2
= L\cdot \frac{\eta_{T-1}}{2}(1-\beta^{T})\|\mathbf{1}\|_2
= \frac{L}{2}(1-\beta^{T})\,\eta_{T-1}\,\sqrt{d}.
\]
Moreover, \eqref{eq:unrolled-stoch} decreases each coordinate strictly while staying positive, hence $\|x_t\|_2$
(and thus $\|\nabla F(x_t)\|_2$) is strictly decreasing in $t$, so the minimum over $t\in\{0,\ldots,T\}$ is attained at $T$.
Finally, requiring $\|\nabla F(x_T)\|_2\le \varepsilon$ yields \eqref{eq:eta-necessary-stoch}.
\end{proof}

\begin{theorem}
\label{thm:laprop_omega_d_eps2}
Take $\eta_t=\frac{c}{L\sqrt{t+1}}$ for a constant $c>0$, and assume $T$ is large enough so that
$1-\beta^{T}\ge \tfrac12$. Then there exists an initialization $x_0$ (given by \eqref{eq:a-def-stoch}) such that
the recursion \eqref{eq:laprop-sign-stoch} satisfies
\[
\min_{0\le t\le T}\|\nabla F(x_t)\|_2 \;\ge\; \frac{1}{4}\cdot \frac{c\sqrt{d}}{\sqrt{T}}.
\]
Hence, to guarantee $\min_{0\le t\le T}\|\nabla F(x_t)\|_2\le \varepsilon$, it is necessary that
\[
T \;\ge\; \Omega\!\left(\frac{d}{\varepsilon^2}\right).
\]
\end{theorem}

\begin{proof}
Apply Lemma~\ref{lem:laprop_finite_horizon} with $\eta_{T-1}=\frac{c}{L\sqrt{T}}$ and $1-\beta^T\ge \tfrac12$:
\[
\min_{0\le t\le T}\|\nabla F(x_t)\|_2
= \frac{L}{2}(1-\beta^{T})\,\eta_{T-1}\,\sqrt{d}
\ge \frac{L}{2}\cdot \frac12 \cdot \frac{c}{L\sqrt{T}}\sqrt{d}
= \frac{1}{4}\cdot\frac{c\sqrt{d}}{\sqrt{T}}.
\]
Requiring the right-hand side to be at most $\varepsilon$ yields $T=\Omega(d/\varepsilon^2)$.
\end{proof}

\color{black}

\subsection{Proof of Theorem~\ref{thm:mvngrad_convergence}}
\label{app:proof_convergence}

\paragraph{Proof overview.}
The proof relies on three ingredients. First, we use a \emph{virtual-iterate reduction}: defining
$\tilde{x}_t := x_t - \frac{\beta_1}{1-\beta_1}\eta\,u_t$
decouples the momentum recursion and yields the simple update
$\tilde{x}_t-\tilde{x}_{t-1}=-\eta z_t .$

Thus, MVN-Grad with heavy-ball momentum can be analyzed as a momentum-free adaptive method on the virtual sequence, after which the resulting bound is transferred back to the actual iterates $x_t$.

Second, we use a \emph{covariance decomposition} to handle the dependence between the stochastic gradient and the variance-based normalizer. Writing
\[
    \E[z_{t,i}\mid\mathcal F_{t-1}]
    =
    \nabla_i F(x_{t-1})\,\E[w_{t,i}\mid\mathcal F_{t-1}]
    +
    \Cov(\xi_{t,i},w_{t,i}\mid\mathcal F_{t-1}),
\]
with $w_{t,i}=(\sqrt{s_{t,i}}+\varepsilon)^{-1}$, isolates the numerator--denominator coupling into a covariance term. We bound this term using the concavity of the square root, in particular
$|\sqrt a-\sqrt b|\le \sqrt{|a-b|}$, which avoids singular behavior when $s_t$ is small.

Third, Young's inequality controls both the covariance residual and the mismatch between gradients evaluated at the virtual and actual iterates. This yields an inner-product lower bound of the form
\[
\left\langle \nabla F(\tilde x_{t-1}),\E_{t-1}[z_t]\right\rangle
\ge
\frac12 S_t-\frac32 R_t-\frac32 Q_t,
\]
where $S_t$ is the weighted gradient norm, $R_t$ is the virtual-iterate mismatch error, and $Q_t$ is the covariance residual. Combining this lower bound with smoothness of $F$ and telescoping the resulting descent inequality proves Theorem~\ref{thm:mvngrad_convergence}.

\begin{proof}
\textbf{Step 1: Virtual iterate reduction.}
Define $\tilde{x}_t := x_t - \frac{\beta_1}{1-\beta_1}\eta\, u_t$. Since $u_0=0$, we have $\tilde{x}_0=x_0$.
We claim $\tilde{x}_t - \tilde{x}_{t-1} = -\eta z_t$ for all $t \ge 1$.
Indeed, from $x_t = x_{t-1} - \eta u_t$ and $u_t = \beta_1 u_{t-1} + (1-\beta_1)z_t$:
\begin{align*}
\tilde{x}_t - \tilde{x}_{t-1}
&= (x_t - x_{t-1}) - \frac{\beta_1\eta}{1-\beta_1}(u_t - u_{t-1}) \\
&= -\eta u_t - \frac{\beta_1\eta}{1-\beta_1}\big[(\beta_1-1)u_{t-1} + (1-\beta_1)z_t\big] \\
&= -\eta u_t + \beta_1\eta u_{t-1} - \beta_1\eta z_t \\
&= -\eta\big[\beta_1 u_{t-1} + (1-\beta_1)z_t\big] + \beta_1\eta u_{t-1} - \beta_1\eta z_t \\
&= -\eta z_t.
\end{align*}

\textbf{Step 2: Bounded normalized gradient.}
Since $\sqrt{s_{t,i}} + \varepsilon \ge \varepsilon$ and $|g_{t,i}| \le G$:
\[
|z_{t,i}| = \frac{|g_{t,i}|}{\sqrt{s_{t,i}}+\varepsilon} \le \frac{G}{\varepsilon},
\qquad
\|z_t\|^2 \le \frac{dG^2}{\varepsilon^2}.
\]

\textbf{Step 3: Descent on virtual iterates.}
By $L$-smoothness applied to the virtual iterates:
\begin{align*}
F(\tilde{x}_t)
&\le F(\tilde{x}_{t-1}) + \langle \nabla F(\tilde{x}_{t-1}), \tilde{x}_t - \tilde{x}_{t-1}\rangle
+ \frac{L}{2}\|\tilde{x}_t - \tilde{x}_{t-1}\|^2 \\
&= F(\tilde{x}_{t-1}) - \eta\langle \nabla F(\tilde{x}_{t-1}), z_t\rangle
+ \frac{L\eta^2}{2}\|z_t\|^2.
\end{align*}

\textbf{Step 4: Inner product in expectation.}
Write $\E_{t-1}[\cdot] := \E[\cdot \mid \mathcal F_{t-1}]$ and
$\Cov_{t-1}(\cdot,\cdot) := \Cov(\cdot,\cdot \mid \mathcal F_{t-1})$.
Since $\tilde x_{t-1}$ is $\mathcal F_{t-1}$-measurable, taking conditional expectations in Step~3 gives
\[
\mathbb{E}_{t-1} [F(\tilde{x}_t)]
\leq F(\tilde{x}_{t-1})
- \eta\,\langle \nabla F(\tilde{x}_{t-1}),\, \mathbb{E}_{t-1}[z_t]\rangle
+ \frac{L\eta^2}{2}\mathbb{E}_{t-1}[\|z_t\|^2].
\]
Let
\[
w_{t,i}:=\frac{1}{\sqrt{s_{t,i}}+\varepsilon},
\qquad
W_{t,i}:=\mathbb{E}_{t-1}[w_{t,i}],
\qquad
c_{t,i}:=\Cov_{t-1}(\xi_{t,i},w_{t,i}).
\]
Then $z_{t,i}=g_{t,i}w_{t,i}$ and, since
$g_{t,i}=\nabla_iF(x_{t-1})+\xi_{t,i}$ with
$\E_{t-1}[\xi_{t,i}]=0$, we have
\begin{equation}\label{eq:ez_decomp}
\mathbb{E}_{t-1}[z_{t,i}]
=
\nabla_iF(x_{t-1})W_{t,i}+c_{t,i}.
\end{equation}

We first record simple bounds on $W_{t,i}$. Since $\|g_t\|_\infty\le G$ a.s.,
we have $\|\nabla F(x_{t-1})\|_\infty\le G$. Moreover, $m_t$ is an EMA of
bounded gradients, so $\|m_t\|_\infty\le G$. Hence
$|g_{t,i}-m_{t,i}|\le 2G$ and therefore $0\le s_{t,i}\le 4G^2$ a.s. Thus
\begin{equation}\label{eq:w_bounds}
\frac{1}{2G+\varepsilon}\le w_{t,i}\le \frac{1}{\varepsilon},
\qquad
\frac{1}{2G+\varepsilon}\le W_{t,i}\le \frac{1}{\varepsilon}.
\end{equation}

We next bound the covariance term. Define
\[
a_{t,i}:=\nabla_iF(x_{t-1}),\qquad
\delta_{t,i}:=a_{t,i}-m_{t-1,i},
\]
and let
\[
\bar s_{t,i}
:=
\beta_2 s_{t-1,i}
+
(1-\beta_2)\beta_1^2\delta_{t,i}^2,
\qquad
\bar w_{t,i}:=\frac{1}{\sqrt{\bar s_{t,i}}+\varepsilon}.
\]
Here $\bar s_{t,i}$ is the value of $s_{t,i}$ at $\xi_{t,i}=0$, and
$\bar w_{t,i}$ is $\mathcal F_{t-1}$-measurable. Since
$\E_{t-1}[\xi_{t,i}\bar w_{t,i}]=0$,
\[
c_{t,i}
=
\mathbb{E}_{t-1} \!\left[\xi_{t,i}(w_{t,i}-\bar w_{t,i})\right].
\]
Also,
\[
|w_{t,i}-\bar w_{t,i}|
\le
\frac{|\sqrt{s_{t,i}}-\sqrt{\bar s_{t,i}}|}{\varepsilon^2}
\le
\frac{\sqrt{|s_{t,i}-\bar s_{t,i}|}}{\varepsilon^2}.
\]
Using
\[
s_{t,i}-\bar s_{t,i}
=
(1-\beta_2)\beta_1^2(2\xi_{t,i}\delta_{t,i}+\xi_{t,i}^2),
\]
we get
\[
|w_{t,i}-\bar w_{t,i}|
\le
\frac{\beta_1\sqrt{1-\beta_2}}{\varepsilon^2}
\sqrt{|2\xi_{t,i}\delta_{t,i}+\xi_{t,i}^2|}.
\]
Therefore, by Cauchy--Schwarz,
\[
|c_{t,i}|
\le
\frac{\beta_1\sqrt{1-\beta_2}}{\varepsilon^2}
\sqrt{
\mathbb{E}_{t-1}[\xi_{t,i}^2]\,
\mathbb{E}_{t-1}[|2\xi_{t,i}\delta_{t,i}+\xi_{t,i}^2|]
}.
\]
Since $|a_{t,i}|\le G$, $|m_{t-1,i}|\le G$, and $|g_{t,i}|\le G$, we have
$|\delta_{t,i}|\le 2G$ and $|\xi_{t,i}|\le 2G$ almost surely. Thus the variance proxy can be bounded using
$\bar\sigma_\xi:=\min\{\sigma_\xi,2G\}$. Since the final bound is monotone in
$\sigma_\xi$, we relabel $\bar\sigma_\xi$ as $\sigma_\xi$ and assume
$\sigma_\xi\le 2G$ without loss of generality.

\[
\mathbb{E}_{t-1}[|2\xi_{t,i}\delta_{t,i}+\xi_{t,i}^2|]
\le
4G\,\mathbb{E}_{t-1}[|\xi_{t,i}|]+\mathbb{E}_{t-1}[\xi_{t,i}^2]
\le
4G\sigma_\xi+\sigma_\xi^2
\le
6G\sigma_\xi.
\]
Thus
\begin{equation}\label{eq:cov_bound}
|c_{t,i}|
\le
\frac{\beta_1\sqrt{6G}\,(1-\beta_2)^{1/2}\sigma_\xi^{3/2}}{\varepsilon^2}.
\end{equation}

Now we lower bound the descent inner product. Let
\[
a_t:=\nabla F(x_{t-1}),\qquad
b_t:=\nabla F(\tilde x_{t-1}),\qquad
e_t:=b_t-a_t.
\]
Using~\eqref{eq:ez_decomp},
\[
\langle b_t,\mathbb{E}_{t-1}[z_t]\rangle
=
\sum_i (a_{t,i}+e_{t,i})(a_{t,i}W_{t,i}+c_{t,i}).
\]
Define
\[
S_t:=\sum_i a_{t,i}^2W_{t,i},
\qquad
R_t:=\sum_i e_{t,i}^2W_{t,i},
\qquad
Q_t:=\sum_i \frac{c_{t,i}^2}{W_{t,i}}.
\]
Then, by Young's inequality,
\[
\sum_i e_{t,i}a_{t,i}W_{t,i}
\ge
-\frac14 S_t-R_t,
\]
\[
\sum_i a_{t,i}c_{t,i}
\ge
-\frac14 S_t-Q_t,
\]
and
\[
\sum_i e_{t,i}c_{t,i}
\ge
-\frac12 R_t-\frac12 Q_t.
\]
Combining these three bounds gives
\begin{equation}\label{eq:inner_product_lower}
\langle b_t,\mathbb{E}_{t-1}[z_t]\rangle
\ge
\frac12 S_t-\frac32 R_t-\frac32 Q_t.
\end{equation}

It remains to bound $R_t$ and $Q_t$. By $L$-smoothness,
\[
\|e_t\|
\le
L\|\tilde x_{t-1}-x_{t-1}\|
=
\frac{L\beta_1\eta}{1-\beta_1}\|u_{t-1}\|.
\]
Since $u_{t-1}$ is an EMA of vectors $z_k$ satisfying
$\|z_k\|\le G\sqrt d/\varepsilon$, we have
$\|u_{t-1}\|\le G\sqrt d/\varepsilon$. Therefore,
\[
\|e_t\|
\le
\frac{L\beta_1\eta G\sqrt d}{(1-\beta_1)\varepsilon}.
\]
Using $W_{t,i}\le 1/\varepsilon$,
\begin{equation}\label{eq:R_bound}
R_t
\le
\frac{1}{\varepsilon}\|e_t\|^2
\le
\frac{L^2\beta_1^2\eta^2 dG^2}{(1-\beta_1)^2\varepsilon^3}.
\end{equation}
Similarly, using~\eqref{eq:w_bounds} and~\eqref{eq:cov_bound},
\begin{equation}\label{eq:Q_bound}
Q_t
\le
\frac{6\beta_1^2 d(1-\beta_2)G(2G+\varepsilon)\sigma_\xi^3}{\varepsilon^4}.
\end{equation}
In particular, if $\varepsilon\le G$, then
\[
Q_t
\le
\frac{18\beta_1^2 d(1-\beta_2)G^2\sigma_\xi^3}{\varepsilon^4}.
\]

Substituting~\eqref{eq:inner_product_lower},~\eqref{eq:R_bound}, and~\eqref{eq:Q_bound}
into the conditional descent inequality, and using
$\E_{t-1}[\|z_t\|^2]\le dG^2/\varepsilon^2$, gives
\begin{align}
\mathbb{E}_{t-1}[F(\tilde{x}_t)]
&\le
F(\tilde{x}_{t-1})
-\frac{\eta}{2}S_t
+
\frac{3\eta}{2}R_t
+
\frac{3\eta}{2}Q_t
+
\frac{L\eta^2 dG^2}{2\varepsilon^2}  \nonumber\\
&\le
F(\tilde{x}_{t-1})
-\frac{\eta}{2}S_t
+
\frac{3L^2\beta_1^2\eta^3 dG^2}{2(1-\beta_1)^2\varepsilon^3}
+
\frac{9\beta_1^2\eta d(1-\beta_2)G(2G+\varepsilon)\sigma_\xi^3}{\varepsilon^4}
+
\frac{L\eta^2 dG^2}{2\varepsilon^2}.
\label{eq:one_step_clean}
\end{align}
Finally, by~\eqref{eq:w_bounds},
\[
S_t
=
\sum_i \nabla_iF(x_{t-1})^2 W_{t,i}
\ge
\frac{1}{2G+\varepsilon}\|\nabla F(x_{t-1})\|^2.
\]

\textbf{Step 5: Telescope.}
Taking total expectations in~\eqref{eq:one_step_clean} and summing from
$t=1$ to $T$, we obtain
\begin{align*}
\frac{\eta}{2}\sum_{t=1}^{T}\E[S_t]
&\le
F(\tilde x_0)-\E[F(\tilde x_T)]
+
\frac{3L^2\beta_1^2\eta^3T dG^2}
{2(1-\beta_1)^2\varepsilon^3}  \\
&\quad
+
\frac{9\beta_1^2\eta T d(1-\beta_2)G(2G+\varepsilon)\sigma_\xi^3}
{\varepsilon^4}
+
\frac{L\eta^2T dG^2}{2\varepsilon^2}.
\end{align*}
Since $\tilde x_0=x_0$ and $F(\tilde x_T)\ge F_\star$, this gives
\begin{align*}
\frac{\eta}{2}\sum_{t=1}^{T}\E[S_t]
&\le
\Delta_0
+
\frac{3L^2\beta_1^2\eta^3T dG^2}
{2(1-\beta_1)^2\varepsilon^3}  \\
&\quad
+
\frac{9\beta_1^2\eta T d(1-\beta_2)G(2G+\varepsilon)\sigma_\xi^3}
{\varepsilon^4}
+
\frac{L\eta^2T dG^2}{2\varepsilon^2}.
\end{align*}
Moreover, by~\eqref{eq:w_bounds},
\[
S_t
=
\sum_i \nabla_iF(x_{t-1})^2 W_{t,i}
\ge
\frac{1}{2G+\varepsilon}\|\nabla F(x_{t-1})\|^2.
\]
Therefore,
\begin{align*}
\frac1T\sum_{t=0}^{T-1}\E\|\nabla F(x_t)\|^2
&\le
\frac{2(2G+\varepsilon)\Delta_0}{\eta T}
+
(2G+\varepsilon)L\eta\frac{dG^2}{\varepsilon^2} \\
&\quad
+
\frac{3(2G+\varepsilon)L^2\beta_1^2\eta^2 dG^2}
{(1-\beta_1)^2\varepsilon^3}
+
\frac{18\beta_1^2d(1-\beta_2)G(2G+\varepsilon)^2\sigma_\xi^3}
{\varepsilon^4}.
\end{align*}
Under $\varepsilon\le G$, we have $2G+\varepsilon\le 3G$. Also, by the
stepsize condition
$\eta \le c\frac{(1-\beta_1)^2\varepsilon}{L\beta_1^2},$
for $c>0$ sufficiently small, the third term is dominated by the second
term. Hence there exist constants $C_1,C_2>0$, depending only on
$\beta_1$, such that
\[
\frac1T\sum_{t=0}^{T-1}\E\|\nabla F(x_t)\|^2
\le
\frac{C_1G\Delta_0}{\eta T}
+
C_1L\eta\frac{dG^3}{\varepsilon^2}
+
\frac{C_2d(1-\beta_2)G^3\sigma_\xi^3}{\varepsilon^4}.
\]
This proves~\eqref{eq:general_rate}.
\end{proof}

\newpage

\newpage

\newpage

\section{Experimental details}

\subsection{Compute resources}

CIFAR-100/ResNet-18 experiments were run on NVIDIA H100 NVL and NVIDIA GH200 GPUs, depending on availability. The batch-size 128 runs used 1 NVIDIA H100 NVL GPUs and took approximately 4000 seconds per run. The batch-size 1024 runs took approximately 4000--4300 seconds per completed run, using either 1 NVIDIA H100 NVL GPUs or 1 NVIDIA GH200 GPU.

WikiText-103/GPT-2 30M experiments were run on NVIDIA H100 GPUs. Each run used batch size 128, sequence length 512, and 100k training steps, and took approximately 4 wall-clock hours based on the logged throughput.

OpenWebText/GPT-2 124M experiments were run on NVIDIA GH200 GPUs. The selected-configuration runs took approximately 46k--48k seconds per run, corresponding to roughly 13 wall-clock hours. MVN-GradW has one additional optimizer-state tensor relative to AdamW / AdaBeliefW / LaPropW; in our OpenWebText runs, wall-clock time remained comparable across optimizers. The full research process required additional compute beyond the final reported runs due to pilot experiments, failed configurations, and hyperparameter sweeps.

\subsection{MNIST experiment hyperparameters}
\label{app:mnist-hparams}

\paragraph{Fig.~\ref{fig:loss_smallbatch}.}
\[
\eta = 5 \cdot 10^{-3}, \qquad 
\beta_1 = 0.999, \qquad 
\beta_2 = 0.7, \qquad 
\varepsilon = 10^{-8}, \qquad
\text{batch size} = 128.
\]

\paragraph{Fig.~\ref{fig:loss_largebatch}.}
\[
\eta = 1 \cdot 10^{-4}, \qquad 
\beta_1 = 0.95, \qquad 
\beta_2 = 0.999, \qquad 
\varepsilon = 10^{-8}, \qquad
\text{batch size} = 1024.
\]

No optimizer-specific tuning was performed in either panel.

\subsection{MNIST conditional-variance diagnostic}
\label{app:mnist-vgap}

For Figure~\ref{fig:thm_vgap_ada_prop}, we estimate the one-step conditional
update variance along a real training trajectory. We train a two-hidden-layer
MLP on MNIST using AdaBelief and MVN-Grad with batch size $16$, learning rate
$10^{-3}$, $\beta_1=0.9$, and $\beta_2=0.95$. At each selected checkpoint, we
freeze the optimizer state $(x_{t-1},m_{t-1},s_{t-1})$ and evaluate both
methods from the same state by drawing $K=128$ independent mini-batches. Using
the resulting gradients, we form one-step updates and estimate
\[
\Delta\Var(t)
=
\Var\!\left(\Delta_t^{\mathrm{AB}}\mid\mathcal F_{t-1}\right)
-
\Var\!\left(\Delta_t^{\mathrm{MV}}\mid\mathcal F_{t-1}\right)
\]
via Monte Carlo. Variances are computed coordinate-wise, averaged over
parameters, and then averaged over three seeds.

\subsection{Single-spike robustness experiment details}
\label{app:ssr-hparams}

Figures~\ref{fig:thm_ssr_1_1}--\ref{fig:thm_ssr_1_2} simulate the single-spike model from
Theorem~\ref{thm:spike}. We set $g_0 = Mu$ for fixed $u\neq 0$, followed by
$g_t=u$ for all $t\ge 1$, and initialize the adaptive statistics with
$m_{-1}=0$ and $v_{-1}=s_{-1}=\bar d>0$. For each configuration, we simulate
$T$ iterations and record the peak update magnitude
\[
\max_{0\le \tau\le T}|\Delta_\tau|
\]
as a function of the spike size $M$. Within each panel, all optimizers use
identical hyperparameters; only $M$ varies across runs. A formal spike analysis
for AdaBelief is more involved because its variance estimator is coupled to the
momentum state, so we report its behavior empirically.

\paragraph{Figure~\ref{fig:thm_ssr_1_1} (long horizon).}
\[
\beta_1 = 0.9, \qquad
\beta_2 = 0.6, \qquad
\varepsilon = 10^{-8}, \qquad
\bar d = 1, \qquad
u = 10^{-3}, \qquad
T = 1000.
\]

\paragraph{Figure~\ref{fig:thm_ssr_1_2} (early carry-over regime).}
\[
\beta_1 = 0.99999, \qquad
\beta_2 = 0.1, \qquad
\varepsilon = 10^{-8}, \qquad
\bar d = 10, \qquad
u = 10, \qquad
T = 50.
\]

No optimizer-specific tuning was performed in either panel.

\subsection{CIFAR-100 with ResNet-18}
\label{app:hparams-cifar}
\paragraph{Hyperparameter search.}
We use the same search budget across methods per batch size.
For batch size 1024, we first swept $\eta \in {2\times10^{-3}, 3\times10^{-3}, 4\times10^{-3}}$ with $\beta_1 = 0.9$ and $\beta_2 \in {0.99, 0.999}$. We then swept $\beta_1 \in {0.7, 0.8, 0.9, 0.95}$, and $\beta_2 \in {0.99, 0.999}$ at $\eta \in {10^{-4}, 10^{-3}, 4\times10^{-3}, 10^{-2}}$. 
For batch size 128, we swept $\eta \in \{10^{-4}, 5\!\times\!10^{-4}, 10^{-3}\}$ with $\beta_1=0.9$ and $\beta_2\in\{0.99,0.999\}$; for AdaBelief and MVN-Grad we additionally evaluate $\varepsilon_s \in \{10^{-6},10^{-8}\}$, while using $\varepsilon=10^{-8}$ for all methods.
We then swept $\beta_1 \in \{0.7,0.8,0.9,0.95\}$ over $\eta \in \{5\!\times\!10^{-4},10^{-3},5\!\times\!10^{-3},10^{-2}\}$.
In both cases, we perform single-seed selection followed by 3-seed evaluation.

\paragraph{Selected hyperparameters for CIFAR-100 / ResNet-18} For reproducibility, Tables~\ref{tab:cifar-bs128}--\ref{tab:cifar-bs1024} report results for the
best configuration (by validation / selected checkpoint criterion) found in our sweep for each optimizer.
\begin{table}[h]
\centering
\caption{Selected optimizer hyperparameters for CIFAR-100 / ResNet-18 (batch size 128).}
\label{tab:hparams-cifar-bs128}
\begin{tabular}{lcccc}
\toprule
\textbf{Optimizer} & $\eta$ & $\beta_1$ & $\beta_2$ & $\varepsilon_s$ \\
\midrule
LaProp     & $5\!\times\!10^{-4}$ & $0.7$  & $0.999$ & -  \\
AdaBelief  & $1\!\times\!10^{-3}$ & $0.9$  & $0.99$  & $10^{-8}$ \\
Adam       & $5\!\times\!10^{-4}$ & $0.8$  & $0.999$ & - \\
MVN-Grad   & $1\!\times\!10^{-3}$ & $0.95$ & $0.99$ & $10^{-8}$ \\
\bottomrule
\end{tabular}
\end{table}

\begin{table}[h]
\centering
\caption{Selected optimizer hyperparameters for CIFAR-100 / ResNet-18 (batch size 1024).}
\label{tab:hparams-cifar-bs1024}
\begin{tabular}{lcccc}
\toprule
\textbf{Optimizer} & $\eta$ & $\beta_1$ & $\beta_2$ & $\varepsilon_s$ \\
\midrule
AdaBelief  & $4\!\times\!10^{-3}$ & $0.7$  & $0.99$  & $10^{-8}$ \\
Adam       & $1\!\times\!10^{-3}$ & $0.8$  & $0.999$ & -  \\
LaProp     & $1\!\times\!10^{-3}$ & $0.95$ & $0.999$ & -  \\
MVN-Grad   & $4\!\times\!10^{-3}$ & $0.7$  & $0.99$  & $10^{-8}$ \\
\bottomrule
\end{tabular}
\end{table}

\subsection{Language-modeling hyperparameter grids and selected configurations}
\label{app:hparams-lm}

This appendix lists the tuning grids and selected optimizer hyperparameters
corresponding to the language-modeling results in Table~\ref{tab:GPT-Wiki}
and Table~\ref{tab:GPT-Open-loss}. All other settings follow the training
setup described in the main text.

\noindent\textbf{WikiText-103 sweep.}
We use a two-stage tuning protocol. First, we perform a coarse sweep over
learning rates $\{10^{-4},10^{-3},10^{-2}\}$ with
$\beta_1\in\{0.9,0.95\}$ and $\beta_2\in\{0.99,0.999\}$ to identify stable
learning-rate regimes. Final comparisons use a focused sweep in which all
optimizers are evaluated on identical grids with learning rates
$\{5\times 10^{-4},10^{-3}\}$, $\beta_1\in\{0.6,0.7,0.9,0.95\}$, and
$\beta_2\in\{0.95,0.99\}$. Variance-normalized methods additionally evaluate
$\varepsilon_s\in\{0,10^{-8}\}$.

\begin{table}[ht]
\centering
\caption{WikiText-103 selected optimizer hyperparameters used in Table~\ref{tab:GPT-Wiki}. A dash indicates that the method does not use $\varepsilon_s$.}
\label{tab:hparams-gpt-wiki}
\begin{tabular}{lcccc}
\toprule
\textbf{Optimizer} & $\eta$ & $\beta_1$ & $\beta_2$ & $\varepsilon_s$ \\
\midrule
Adam      & $1\times 10^{-3}$ & $0.95$ & $0.99$ & -- \\
LaProp    & $1\times 10^{-3}$ & $0.95$ & $0.99$ & -- \\
AdaBelief & $1\times 10^{-3}$ & $0.95$ & $0.99$ & $0$ \\
MVN-Grad  & $1\times 10^{-3}$ & $0.95$ & $0.99$ & $0$ \\
\bottomrule
\end{tabular}
\end{table}

\noindent\textbf{OpenWebText sweep.}
Based on pilot runs, we fix the peak learning rate to $\eta=10^{-4}$ for all
optimizers and tune only $(\beta_1,\beta_2)$. For variance-normalized methods
we set $\varepsilon_s=0$. For each optimizer, we evaluate the same nine
representative $(\beta_1,\beta_2)$ pairs:
$(0.6,0.99)$, $(0.6,0.999)$, $(0.7,0.95)$, $(0.7,0.99)$, $(0.7,0.999)$,
$(0.8,0.9)$, $(0.8,0.99)$, $(0.8,0.999)$, and $(0.9,0.95)$.
We select the configuration minimizing final validation loss and re-run it
with three seeds.

\begin{table}[ht]
\centering
\caption{OpenWebText selected optimizer hyperparameters used in Table~\ref{tab:GPT-Open-loss}. A dash indicates that the method does not use $\varepsilon_s$.}
\label{tab:hparams-gpt-owt}
\begin{tabular}{lcccc}
\toprule
\textbf{Optimizer} & $\eta$ & $\beta_1$ & $\beta_2$ & $\varepsilon_s$ \\
\midrule
AdamW      & $1\times 10^{-4}$ & $0.6$ & $0.99$  & -- \\
LaPropW    & $1\times 10^{-4}$ & $0.6$ & $0.999$ & -- \\
AdaBeliefW & $1\times 10^{-4}$ & $0.6$ & $0.999$ & $0$ \\
MVN-GradW  & $1\times 10^{-4}$ & $0.6$ & $0.999$ & $0$ \\
\bottomrule
\end{tabular}
\end{table}

\newpage

\newpage

\printbibliography

\end{document}